\pdfoutput=1
\documentclass[letterpaper]{article}
\usepackage{ijcai18}
\usepackage{times}
\usepackage{latexsym}
\usepackage{amsmath}
\usepackage{amssymb}
\usepackage{amsthm}
\usepackage{paralist}
\usepackage{thmtools}
\usepackage{thm-restate} 
\usepackage{macros}
\usepackage{chngcntr}
\usepackage[draft]{ifdraft}
\allowdisplaybreaks
\usepackage[utf8]{inputenc}
\usepackage[small]{caption}

\usepackage[ruled,linesnumbered]{algorithm2e}

\SetAlFnt{\small}
\SetAlCapFnt{\small}
\SetAlCapNameFnt{\small}
\SetAlCapHSkip{0pt}
\IncMargin{-\parindent}
\SetKwInput{KwPar}{Parameter}

\newtheorem{theorem}{Theorem}

\newtheorem{lemma}[theorem]{Lemma}
\newtheorem{definition}[theorem]{Definition}

\newcommand{\newmarkedtheorem}[1]{%
  \newenvironment{#1}
    {\pushQED{\qed}\csname inner@#1\endcsname%
     \renewcommand{\qedsymbol}{$\lhd$}\normalfont}
    {\popQED\csname endinner@#1\endcsname}%
  \newtheorem{inner@#1}%
}
\newmarkedtheorem{example}[theorem]{Example}

\title{ 
    Stratified Negation in Limit Datalog Programs
}
\author{
    Mark Kaminski,~~Bernardo Cuenca Grau,~~Egor V. Kostylev,~~Boris Motik\and Ian Horrocks\\
    Department of Computer Science, University of Oxford, UK\\
    \{mark.kaminski,~bernardo.cuenca.grau,~egor.kostylev,~boris.motik,~ian.horrocks\}@cs.ox.ac.uk
}
 
\begin{document}

\maketitle
  
\begin{abstract}
  There has recently been an increasing interest in declarative data analysis,
  where analytic tasks are specified using a logical language, and their
  implementation and optimisation are delegated to a general-purpose query
  engine.  Existing declarative languages for data analysis can be formalised
  as variants of logic programming equipped with arithmetic function symbols
  and/or aggregation, and are typically undecidable. In prior work,
  the language of \emph{limit programs} was proposed, which is sufficiently
  powerful to capture many analysis tasks and has decidable entailment
  problem. Rules in this language, however, do not allow for negation. In this
  paper, we study an extension of limit programs with stratified
  negation-as-failure. We show that the additional expressive power makes
  reasoning computationally more demanding, and provide tight data complexity
  bounds. We also identify a fragment with tractable data complexity and
  sufficient expressivity to capture many relevant tasks.
\end{abstract}

\section{Introduction}\label{sec:introduction}

Data analysis tasks are becoming increasingly important in information systems.
Although these tasks are currently implemented using code written in standard programming languages, in recent years there has been
a significant shift towards declarative solutions where the definition of the task is clearly separated from its
implementation \cite{DBLP:conf/eurosys/AlvaroCCEHS10,DBLP:journals/pvldb/Markl14,DBLP:journals/tkde/SeoGL15,DBLP:journals/pvldb/WangBH15,DBLP:conf/sigmod/ShkapskyYICCZ16,DBLP:conf/ijcai/KaminskiGKMH17}.

Languages for declarative data analysis are typically rule-based, and they have already been implemented in 
reasoning engines such as BOOM
\cite{DBLP:conf/eurosys/AlvaroCCEHS10}, DeALS
\cite{DBLP:conf/sigmod/ShkapskyYICCZ16}, Myria
\cite{DBLP:journals/pvldb/WangBH15}, SociaLite
\cite{DBLP:journals/tkde/SeoGL15}, Overlog
\cite{DBLP:journals/cacm/LooCGGHMRRS09}, Dyna
\cite{DBLP:conf/datalog/EisnerF10}, and Yedalog
\cite{DBLP:conf/snapl/ChinDEHMOOP15}.

Formally, such declarative languages can be seen as variants of 
logic programming equipped with means for capturing quantitative aspects of the data, such as 
arithmetic function symbols and aggregates. It is, however, well-known since the 
'90s that the combination of recursion with numeric computations in rules easily leads to 
semantic difficulties \cite{DBLP:conf/vldb/MumickPR90,DBLP:conf/slp/KempS91,DBLP:journals/jlp/BeeriNST91,DBLP:conf/pods/Gelder92,DBLP:journals/tcs/ConsensM93,DBLP:journals/jcss/GangulyGZ95,RossS97,DBLP:journals/vldb/MazuranSZ13},
and/or undecidability of reasoning   \cite{DBLP:journals/csur/DantsinEGV01,DBLP:conf/ijcai/KaminskiGKMH17}.
In particular, undecidability
carries over to the languages underpinning the aforementioned reasoning engines for 
data analysis.

\citeA{DBLP:conf/ijcai/KaminskiGKMH17} have recently proposed the language of
\emph{limit Datalog programs}---a decidable variant of negation-free Datalog equipped with
arithmetic functions over the integers that is expressive enough to
capture many data analysis tasks.
The key feature of limit programs is that 
all intensional predicates with a numeric argument are
\emph{limit predicates}, the extension of which represents 
minimal ($\tmin$) or maximal
($\tmax$) bounds of numeric values.
For instance, if we encode a weighted directed graph as facts over a ternary 
$edge$ predicate and a unary $node$ predicate in the obvious way, then 
the following rules encode the all-pairs shortest path problem, where
the ternary $\tmin$ limit predicate $d$ is used to encode the distance  from any node to any other node
in the graph as the length of a shortest path between them.
\begin{align}
  \mathit{node}(x)&\to\mathit{d}(x,x,0) \label{eq:motivation-1}\\
    \mathit{d}(x,y,m)\land\mathit{edge}(y,z,n)&\to\mathit{d}(x,z,m+n)  \label{eq:motivation-2}
\end{align}
The semantics of $\tmin$ predicates is defined such that a fact $d(u,v,k)$ is entailed from these rules and a dataset if and only if 
the distance from $u$ to $v$ is at most $k$; as a result,
all facts  $d(u,v,k')$ with $k' \geq k$ are also entailed. This is in contrast to
standard first order predicates, where there is no semantic relationship between
$d(u,v,k)$ and $d(u,v,k')$.
The intended semantics of limit predicates can be axiomatised 
using rules over standard predicates; in particular, our example limit program 
is equivalent to a standard logic program consisting of rules \eqref{eq:motivation-1}, \eqref{eq:motivation-2},
and the following rule \eqref{eq:axiomatisation-motivation}, where $d$ is now treated as a regular first-order predicate:
\begin{align}
d(x,y,k) \wedge (k \leq k') & \to d(x,y,k'). \label{eq:axiomatisation-motivation}
\end{align}
\citeA{DBLP:conf/ijcai/KaminskiGKMH17} showed that, under certain restrictions on the use of multiplication,
reasoning (i.e., fact entailment) over limit programs is decidable and $\conp$-complete in data
 complexity; then, they proposed a practical fragment with tractable
data complexity.

Limit Datalog programs as defined in prior work are, however, positive and
hence do not allow for negation-as-failure in the body of rules. Non-monotonic
negation applied to limit atoms can be useful, not only to express a wider
range of data analysis tasks, but also to declaratively obtain solutions to
problems where the cost of such solutions is defined by a positive limit
program. For instance, our example limit program consisting of
rules~\eqref{eq:motivation-1} and~\eqref{eq:motivation-2} provides the length
of a shortest path between any two nodes, but does not provide access to any of
the paths themselves---an issue that we will be able to solve using
non-monotonic negation.

In this paper, we study the language of limit programs with stratified
negation-as-failure.  Our language extends both positive limit Datalog as
defined in prior work and plain (function-free) Datalog with stratified
negation.  We argue that our language provides useful additional expressivity,
but at the expense of increased complexity of reasoning; for programs with
restricted use of multiplication, complexity jumps from $\conp$-completeness in
the case of positive programs, to $\deltatwop$-completeness for programs with
stratified negation.  We also show that the tractable fragment of positive
limit programs defined in \cite{DBLP:conf/ijcai/KaminskiGKMH17} can be
seamlessly extended with stratified negation while preserving tractability of
reasoning; furthermore, the extended fragment is sufficiently expressive to
capture the relevant data analysis tasks.

The proofs of all our results are given in \ifdraft{the appendix}{%
  a technical report (see https://www.dropbox.com/s/ci2b1hltb1hlcr6/paper.pdf)%
}.


\section{Preliminaries}\label{sec:preliminaries}

In this section we recapitulate the syntax and semantics of Datalog programs with 
integer arithmetic and stratified negation (see e.g., 
\cite{DBLP:journals/csur/DantsinEGV01} for an excellent survey).

\myparagraph{Syntax} We assume a fixed vocabulary of countably
infinite, mutually disjoint sets of \emph{predicates} equipped with non-negative arities, \emph{objects},
\emph{object variables}, and \emph{numeric variables}. 
Each position ${1 \leq i \leq n}$ of an $n$-ary predicate
is of either \emph{object} or \emph{numeric sort}. An \emph{object term} is an object
or an object variable. A \emph{numeric term} is an integer, a numeric variable,
or of the form ${s_1 + s_2}$, ${s_1 - s_2}$, or ${s_1 \times s_2}$ where $s_1$
and $s_2$ are numeric terms and $+$, $-$, and $\times$ are the standard
\emph{arithmetic functions}. A \emph{constant} is an object or an integer. 
A \emph{standard atom} is
of the form ${B(t_1, \dots, t_n)}$, with $B$ an $n$-ary predicate  and
each $t_i$ a term matching the sort of the $i$-th position of $B$. 
A \emph{(standard) positive literal} is a standard atom, and a \emph{(standard) negative literal} is of the form $\naf\alpha$, for $\alpha$ a standard atom. 
A \emph{comparison atom} is of the form ${(s_1 < s_2)}$ or
${(s_1 \leq s_2)}$, with $<$ and $\leq$ the usual \emph{comparison
predicates} over the integers, and $s_1$ and $s_2$ numeric terms. We
write $(s_1\doteq s_2)$ as an abbreviation for $(s_1\le s_2)\land(s_2\le s_1)$.
A term, atom or literal is \emph{ground} if it has no variables. 

A \emph{rule} $r$ has the form
${\bigwedge\nolimits_i \mu_i \wedge \bigwedge\nolimits_j \beta_j \to \alpha}$,
where the \emph{body} $\bigwedge\nolimits_i \mu_i \wedge \bigwedge\nolimits_j \beta_j$ is a possibly empty conjunction of standard literals $\mu_{i}$ and comparison atoms $\beta_j$, and the \emph{head} $\alpha$ is a standard atom.
We 
assume without loss of generality that standard body literals are function-free; indeed, a conjunction with a functional term $s$ can be equivalently rewritten by replacing $s$ with a fresh variable $x$
and adding $(x \doteq s)$ to the conjunction.
A rule $r$ is \emph{safe} if each object variable in $r$ occurs in a positive
literal in the body of $r$.  A \emph{ground 
  instance of}  $r$ is obtained from
$r$ by substituting 
each variable by a constant of the right sort.

A \emph{fact} is a rule with empty body and a function-free standard atom in the head that has no variables in object positions and no repeated variables in numeric positions. Intuitively, a variable in a fact says that the fact holds for every integer in the position. As a convention, we will omit $\to$ and use symbol $\allZ$ instead of variables when writing facts.
A \emph{dataset} $\Dat$ is a finite
set of facts. Dataset $\Dat$ is \emph{ordered} if
\begin{inparaenum}[\it (i)]
\item it contains facts $\mathit{first}(a_1)$, $\mathit{next}(a_1,a_2)$, $\dots$, $\mathit{next}(a_{n-1},a_n)$, $\mathit{last}(a_n)$ for some
  repetition-free enumeration $a_1,\dots,a_n$ of all objects in $\Dat$; and
\item it contains no other facts over predicates
  $\mathit{first}$, $\mathit{next}$, and $\mathit{last}$. 
\end{inparaenum}
A \emph{ program} is a finite set of safe rules; 
without loss of generality we assume that distinct rules do not share variables.  A predicate $B$ is
\emph{intensional} (\emph{IDB}) in a program $\Prog$ if $B$ occurs in $\Prog$
in the head of a rule that is not a fact; otherwise, $B$ is \emph{extensional}
(\emph{EDB}) in $\Prog$. 
Program $\Prog$ is \emph{positive} if it has no negative literals, and it is
\emph{semi-positive} if negation occurs only  in front of EDB atoms. A
\emph{stratification} of $\Prog$ is a function $\lambda$ mapping each predicate
to a positive integer such that, for each rule with the head over a predicate $A$ and each
standard body literal $\mu$ over $B$, we have $\lambda(B)\le\lambda(A)$ if $\mu$ is positive, and
$\lambda(B)<\lambda(A)$ if $\mu$ is negative.
Program $\Prog$ is \emph{stratified} if it admits a stratification.
Given a stratification $\lambda$, we write $\Prog[i]$ for the \emph{$i$-th stratum of $\Prog$ over $\lambda$}---that is, the set of all rules in $\Prog$ whose head predicates $A$ satisfy
$\lambda(A)=i$. Note that each stratum is a semi-positive program. 

\myparagraph{Semantics} A \emph{(Herbrand) interpretation} $I$ is a possibly
infinite set of ground facts (i.e., facts without $\allZ$). Interpretation $I$ \emph{satisfies} a
ground atom $\alpha$, written ${I \models \alpha}$, if either
\begin{inparaenum}[\it (i)]
\item $\alpha$ is a standard atom such that evaluation of the arithmetic functions in
  $\alpha$ under the usual semantics over integers produces a fact in $I$; or
\item $\alpha$ is a comparison atom that evaluates to $\true$ under the
 usual semantics.
\end{inparaenum}
Interpretation $I$ \emph{satisfies} a ground negative literal $\naf\alpha$, written
$I\models\naf\alpha$, if $I\not\models\alpha$.  The notion of satisfaction is
extended to conjunctions of ground literals, rules, and programs as
in first-order logic, with all variables in rules implicitly universally quantified.  If
${I}$ satisfies a program $\Prog$, then $I$ is a \emph{model} of $\Prog$.
For $I$ a Herbrand interpretation and $\SProg$ a (possibly infinite) semi-positive set of rules, let $\IFPStep{\SProg}{I}$ be the set of facts $\alpha$ such
that $\varphi \to \alpha$ is a ground instance of a rule in $\SProg$ and $I \models \varphi$. Given a program $\Prog$ and a stratification $\lambda$ of
$\Prog$, for each $i,j \geq 0$ we define interpretation 
$\MatSeq i j$ by induction on $i$ and~$j$:
\begin{align*}
  \MatSeq{0}{j}&=\MatSeq{i}{0}=\emptyset;
  & \MatSeq{i+1}{j+1}&=\IFPStep{\Prog[i+1]\cup\MatSeq{i}{\infty}}{\MatSeq{i+1}{j}};
  & \MatSeq{i}{\infty}&=\bigcup_{j\ge 0}\MatSeq{i}{j}.
\end{align*}
The \emph{materialisation} $\Mat\Prog$ of $\Prog$ is the interpretation
$\MatSeq{k}{\infty}$, for $k$ the greatest number such that
$\Prog[k]\ne\emptyset$. The materialisation of a program
does not depend on the chosen stratification.
%
A stratified program $\Prog$ \emph{entails} a fact $\alpha$, written
$\Prog\models\alpha$, if $\alpha'\in\Mat\Prog$ for every ground instance $\alpha'$ of $\alpha$.
For positive programs, 
this definition coincides with the usual first-order notion of entailment: for
$\Prog$ positive and $\alpha$ a fact, $\Prog\models\alpha$ if and
only if $I\models\alpha$ holds for all $I\models\Prog$.

\myparagraph{Reasoning} We study the computational properties of
checking whether ${\Prog\cup\Dat\models \alpha}$, for $\Prog$ a program, $\Dat$ a
dataset, and $\alpha$ a fact. We are interested in \emph{data complexity}, which
assumes that only $\Dat$ and $\alpha$ form the input while $\Prog$ is
fixed. Unless otherwise stated, all numbers in the input are coded in binary,
and the \emph{size} $\ssize{\Prog}$ of $\Prog$ is the size of its
representation.  Checking ${\Prog \cup \Dat \models \alpha}$ is undecidable even if the
only arithmetic function in $\Prog$ is $+$ \cite{DBLP:journals/csur/DantsinEGV01} and
predicates have at most one numeric position
\cite{DBLP:conf/ijcai/KaminskiGKMH17}.

We use standard definitions of the basic complexity classes such as
$\pt$, $\np$, $\conp{}$, and $\fph{}$. Given a complexity class $C$, $\pt^{C}$ is
the class of decision problems solvable in polynomial time by deterministic Turing machines
with an oracle for a problem in $C$; functional class $\fph{C}$ is defined similarly. Finally,
$\deltatwop$ is a synonym for $\pt^{\np}$.



\section{Stratified Limit Programs}\label{SEC:LIMIT-PROGRAMS}

We introduce 
\emph{stratified limit programs} as a language that
can be seen as either a semantic or a
syntactic restriction of Datalog
with integer arithmetic and stratified negation.
Our language is also an extension of that in~\cite{DBLP:conf/ijcai/KaminskiGKMH17}
with stratified negation. 

\begin{definition}
  A \emph{stratified limit program} is a pair $(\Prog, \tau)$ where
  \begin{compactitem}
  \item[--] $\Prog$ is a stratified program where each
    predicate either has no numeric position, in which case it is an
    \emph{object predicate}, or only its last position is numeric, in which
    case it is a \emph{numeric predicate}, and
  \item[--] $\tau$ is a partial function from numeric predicates to
    $\{\tmin, \tmax\}$ that is total on the IDB predicates in $\Prog$ and on
    predicates occurring in non-ground facts.
  \end{compactitem}
  A numeric predicate $A$ is a $\tmin$ (or $\tmax$) \emph{limit predicate} if
  $\tau(A) = \tmin$ (or $\tau(A) = \tmax$, respectively).  Numeric predicates
  that are not limit predicates are \emph{ordinary}. An atom, fact or literal
  is \emph{numeric}, \emph{limit}, etc.~if so is the used predicate.
\end{definition} 

All notions defined on ordinary Datalog programs $\Prog$ (such as EDB and IDB
predicates, stratification, etc.) transfer to limit programs $(\Prog, \tau)$ by
applying them to $\Prog$.  We often abuse notation and write $\Prog$ instead of
$(\Prog, \tau)$ when $\tau$ is clear from the context or immaterial. Whenever
we consider a union of two limit programs, we silently assume that they
coincide on~$\tau$.  Finally, we denote $\le$ (or $\ge$) by $\preceq_A$ if $A$
is a $\tmax$ (or, respectively, $\tmin$) limit predicate.

Intuitively, a limit fact $B(\vec a,k)$ says that the value of $B$ for a
tuple of objects $\vec{a}$ is $k$ \emph{or more}, if $B$ is $\tmax$, or $k$
\emph{or less}, if $B$ is $\tmin$. 
For example, a $\tmin$ limit fact $d(u,v,k)$ in our
all-pairs shortest path example says that node $v$ is reachable from node $u$ 
via a path with cost $k$ or less. 
The intended semantics of limit predicates can be
axiomatised using standard rules as given next.


\begin{definition}
  An interpretation $I$ \emph{satisfies} a limit program $(\Prog,\tau)$ if it
  satisfies the 
  program $\Prog \cup \ax{\Prog}$,
  where $\ax{\Prog}$ contains the following rule for each limit predicate $A$
  in~$\Prog$:
 \begin{align*}
                                      A(\vec{x},m) \wedge (n \preceq_A m) & \to A(\vec{x},n).
 \end{align*}
The \emph{materialisation} ${\Mat{\Prog,\tau}}$ of $(\Prog,\tau)$ is ${\Mat{\Prog\cup\ax{\Prog}}}$; and $(\Prog,\tau)$ \emph{entails} $\alpha$, written
 $(\Prog,\tau)\models\alpha$, if $\alpha\in\Mat{\Prog,\tau}$. 
\end{definition}

We next demonstrate the use of stratified negation on examples. 
One of the main uses of negation of a limit atom is to `access'
the limit value (e.g., the length of
a shortest path) attained by the atom in the materialisation of previous strata, and then exploit
such values in further computations. 
To facilitate such use of negation in examples, we introduce a
new operator as syntactic sugar in the language.

\begin{definition}
The  \emph{least upper bound expression} 
$\lub{A(\vec s,n)}$ of a $\tmax$ (or $\tmin$) limit atom $A(\vec s,n)$ is the conjunction $A(\vec s,n)\land\naf A(\vec s,m)\land (m\doteq n+t)$ where $t = 1$ (or $t = -1$, respectively) and $m$ is a fresh variable.
\end{definition}

Clearly, $I\models\lub{A(\vec a,k)}$ for $I$ an interpretation and
$A(\vec a,k)$ a ground atom if $k$ is the limit integer such that
$I\models A(\vec a,k)$.

\begin{example}\label{ex:shortest-path}
  An input of the single-pair shortest path problem can be encoded in the
  obvious way as a dataset $\Dat_{\mathit{sp}}$ using a ternary ordinary
  numeric predicate $\mathit{edge}$ to represent the graph's weighted edges,
  and unary facts $\mathit{source}(u)$ and $\mathit{target}(v)$ to identify the
  source and target nodes $u$ and $v$, respectively.  The stratified limit
  program $\Prog_{\mathit{sp}}$ given next computes, together with
  $\Dat_{\mathit{sp}}$ (where all edge weights are positive), a DAG over a binary
  object predicate $\mathit{sp\text{-}edge}$ such that every maximal path in
  the DAG is a shortest path from $u$ to~$v$.
  \begin{align}
    \mathit{source}(x)&\to\mathit{ds}(x,0)\label{eq:sp2:source}\\
    \mathit{ds}(x,m) \land \mathit{edge}(x,y,n) & \to \mathit{ds}(y,m + n)\label{eq:sp2:ind}\\
    \begin{array}{@{}r@{}}
      \lub{\mathit{ds}(x,m_1)}\land\lub{\mathit{ds}(y,m_2)}\\
      \mathit{edge}(x,y,n)\land\mathit{target}(y)\\
      (m_1+n\doteq m_2)
    \end{array}&\begin{array}{@{}l@{}}{}\land{}\\{}\land{}\\{}\to\mathit{sp\text{-}edge}(x,y)\end{array}\label{eq:sp2:path1}\\
    \begin{array}{@{}r@{}}
      \lub{\mathit{ds}(x,m_1)}\land\lub{\mathit{ds}(y,m_2)}\\
     \mathit{edge}(x,y,n)\land\mathit{sp\text{-}edge}(y,z)\\
      (m_1+n\doteq m_2)
    \end{array}&\begin{array}{@{}l@{}}{}\land{}\\{}\land{}\\{}\to\mathit{sp\text{-}edge}(x,y)\end{array}\label{eq:sp2:path2}
  \end{align}
  The first stratum consists of rules \eqref{eq:sp2:source} and \eqref{eq:sp2:ind}, and computes the length of
  a shortest path from $u$ to all other nodes using the $\tmin$ predicate $\mathit{ds}$; in particular,
  $\Prog_{\mathit{sp}}\cup\Dat_{\mathit{sp}}\models\lub{\mathit{ds}(v,k)}$
  if and only if $k$ is the length of a shortest path from $u$ to $v$. Then, in a second stratum, the program computes 
  the $\mathit{sp\text{-}edge}$ predicate such that
  $\Prog_{\mathit{sp}}\cup\Dat_{\mathit{sp}}\models\mathit{sp\text{-}edge}(a,b)$ if and only if the edge
  $(a,b)$ is part of a shortest path from $u$ to~$v$.
\end{example}

\begin{example}\label{ex:closeness-centrality}
  The closeness centrality of a node in a strongly connected weighted directed graph $G$ is a
  measure of how central the node is in the graph~\cite{Sabidussi66}; variants
  of this measure are useful, for instance, for the analysis of market
  potential.  Most commonly, closeness centrality of a node $u$ is defined as
  $1/\sum_{v\textup{ node in }G}d(u,v)$, where $d(u,v)$ is the length of a shortest path from $u$ to $v$; 
  the sum in the denominator is often called the \emph{farness centrality} of $v$. We
  next give a limit program computing a
  node of maximal closeness centrality in a given directed graph. We encode a graph as an ordered dataset
  $\Dat_{\mathit{cc}}$ using, as before, a unary object predicate $\mathit{node}$  and a ternary
  ordinary numeric predicate $\mathit{edge}$.  Program
  $\Prog_{\mathit{cc}}$  consists of rules
  \eqref{eq:cc:dist0}--\eqref{eq:cc:centre}, where $\mathit{d}$,
  $\mathit{fness}'$ and $\mathit{fness}$ are $\tmin$ predicates, and
  $\mathit{centre}'$ and $\mathit{centre}$ are object predicates.
  \begin{align}
    \mathit{node}(x)&\to\mathit{d}(x,x,0)\label{eq:cc:dist0}\\
    \mathit{d}(x,y,m)\land\mathit{edge}(y,z,n)&\to\mathit{d}(x,z,m+n)\label{eq:cc:dist-edge}\\
      \mathit{first}(y) \land \mathit{d}(x,y,n) & \to\mathit{fness}'(x,y,n)\label{eq:cc:fness-first}\\
    \begin{array}{@{}r@{}}
      \mathit{next}(y,z)\\
      \mathit{fness}'(x,y,m)\land\mathit{d}(x,z,n)
    \end{array}
    &\begin{array}{@{}l@{}}{}\land{}\\{}\to\mathit{fness}'(x,z,m\,{+}\,n)\end{array}\label{eq:cc:fness-next}\\
    \mathit{fness}'(x,y,n)\land\mathit{last}(y)&\to\mathit{fness}(x,n)\label{eq:cc:fness}\\
    \mathit{first}(x)&\to\mathit{centre}'(x,x)\label{eq:cc:centre-first}\\
  \begin{array}{@{}r@{}}
    \mathit{next}(x,y)\land\mathit{centre}'(x,z)\\
    \lub{\mathit{fness}(z,n)}\land\lub{\mathit{fness}(y,m)}\\
    (m<n)
    \end{array}
  &\begin{array}{@{}l@{}}{}\land{}\\{}\land{}\\{}\to\mathit{centre}'(y,y)\end{array}\label{eq:cc:centre-next1-tc}\\
  \begin{array}{@{}r@{}}
    \mathit{next}(x,y)\land\mathit{centre}'(x,z)\\
    \lub{\mathit{fness}(z,n)}\land\lub{\mathit{fness}(y,m)}\\
    (n\le m)
  \end{array}
  &\begin{array}{@{}l@{}}{}\land{}\\{}\land{}\\{}\to\mathit{centre}'(y,z)\end{array}\label{eq:cc:centre-next2-tc} \\
    \mathit{centre}'(x,z)\land\mathit{last}(x)&\to\mathit{centre}(z)\label{eq:cc:centre}
  \end{align}
  The first stratum consists of rules  \eqref{eq:cc:dist0}--\eqref{eq:cc:fness}.
 Rules \eqref{eq:cc:dist0} and \eqref{eq:cc:dist-edge} compute the distance (length of a shortest path)
 between any two nodes. Rules \eqref{eq:cc:fness-first}--\eqref{eq:cc:fness} then compute the 
 farness centrality of each node based on the aforementioned distances; for this, the program 
 exploits the order predicates to iterate over the nodes in the graph while recording the best value obtained so far in the iteration 
 using an auxiliary predicate $\mathit{fness}'$. In the second stratum (rules \eqref{eq:cc:centre-first}--\eqref{eq:cc:centre}), 
 the program uses negation to compute the node of minimum farness  centrality (and hence of maximum closeness centrality), which is 
 recorded using the $\mathit{centre}$ predicate; the order is again exploited to iterate over nodes, and an auxiliary predicate $\mathit{centre}'$ is used to record the
 current node of the iteration and the node with the best centrality encountered so far. 
 \end{example}

%


\section{Stratified Limit-Linear Programs}\label{SEC:DECIDABILITY}

By results in~\cite{DBLP:conf/ijcai/KaminskiGKMH17}, checking fact entailment
is undecidable even for positive limit programs. Essentially, this follows from the fact that checking rule applicability over a set of facts requires solving arbitrary non-linear inequalities over integers---that is, solving the 10th Hilbert problem, which is undecidable.
To regain decidability, they proposed a restriction on positive limit
programs, called \emph{limit-linearity}, which ensures that 
every program satisfying the restriction can be
transformed using a grounding technique so that
all numeric terms in the resulting program are linear.
In particular, this implies that
rule applicability can be determined by solving a system of linear inequalities, which is feasible in $\np$. 
As a result, fact entailment for positive limit-linear programs is $\conp$-complete in data complexity.

We next extend the notion of
limit-linearity to programs with stratified negation, and define
semi-grounding as a way to simplify a limit-linear program by
replacing certain types of variables with constants. We then prove that fact entailment is $\deltatwop$-complete in data complexity for such programs.
All  programs in our previous examples are limit-linear
as per the definition given next.

\begin{definition}\label{def:limit-linear}
  A numeric variable $n$ is \emph{guarded} in a rule $r$ of a stratified limit program if
  \begin{compactitem}
  \item[--] either $n$ occurs in a positive ordinary literal in $r$;
  \item[--] or the body of $r$ contains the literals
    $$
    A(\vec s,n_1), \;\; \naf A(\vec s,n_2), \;\; (n_2\doteq n_1+t),
    $$
    where
    $A$ is a $\tmax$ (or $\tmin$) predicate, $t=1$ (or $t=-1$, respectively), and
    $n\in\set{n_1,n_2}$.
  \end{compactitem}

  Rule $r$ is \emph{limit-linear} if each numeric term in $r$ is of the form
  ${s_0 + \sum_{i=1}^n s_i \times m_i}$, where each $m_i$ is a distinct numeric
  variable not occurring in $r$ in a (positive or negative) ordinary numeric
  literal, term $s_0$ uses only variables occurring in a positive ordinary
  literal in $r$, and terms $s_i$ with $i \ge 1$ use only variables that are
  guarded in $r$ and do not use $+$.  A \emph{limit-linear} program contains
  only limit-linear rules.

  A rule $r$ is \emph{semi-ground} if all variables in $r$ are numeric and
  occur only in limit and comparison atoms.
   The \emph{semi-grounding} of a program\/ $\Prog$ is obtained by replacing,
   in every rule $r$ in $\Prog$, each object variable and each numeric variable
   occurring in an ordinary numeric atom in $r$ with a constant in~$\Prog$ in
   all possible ways.
\end{definition}
It is easily seen that the semi-grounding of a limit-linear program $\Prog$ entails
the same facts as $\Prog$ for every dataset. Furthermore, as in prior work, 
Definition \ref{def:limit-linear}
ensures that the semi-grounding of a positive limit-linear program contains
only linear numeric terms; finally, for programs with stratified negation, it
ensures that negation can be eliminated while preserving limit-linearity when
the program is materialised stratum-by-stratum, as we will discuss in detail
later on.

Decidability of fact entailment for positive limit-linear programs is
established by first semi-grounding the program and then reducing fact
entailment over the resulting program to the validity problem of Presburger
formulas \cite{DBLP:conf/ijcai/KaminskiGKMH17}---that is, first-order formulas
interpreted over the integers and composed using only variables, constants $0$
and $1$, functions $+$ and $-$, and the comparisons.

The extension of such a reduction to stratified limit programs, however, is
complicated by the fact that in the presence of negation-as-failure, entailment
no longer coincides with classical first-order entailment. We thus adopt a
different approach, where we show decidability and establish data complexity
upper bounds
according to the following steps.

\emph{Step 1.} We extend the results in \cite{DBLP:conf/ijcai/KaminskiGKMH17} for
  positive programs by showing that, for every positive limit-linear program
  $\Prog$ and dataset $\Dat$, we can compute in $\fpnp{}$ a finite
  representation 
  of its (possibly infinite)
  materialisation $\Mat{\Prog \cup \Dat}$ (see
  Lemma~\ref{lem:pseudo-mat-bounded} and
  Corollary~\ref{cor:lim-lin-mat-fpnp}). This representation is called
  the \emph{pseudo-materialisation} of $\Prog\cup\Dat$.
  
\emph{Step 2.} We further extend the results in Step~1 to semi-positive limit-linear programs,
  where negation occurs only in front of EDB predicates. For this, we show that
  fact entailment for such programs can be reduced in polynomial time 
  in the size of the data to fact entailment over semi-ground positive
  limit-linear programs by exploiting the notion of a \emph{reduct} (see
  Definition~\ref{def:reduct} and Lemma~\ref{lem:reduct}). Thus,
  we can assume existence of an \fpnp{} oracle $O$ for computing the
  pseudo-materialisation of a semi-positive limit-linear program.

\emph{Step 3.} We provide an algorithm (see Algorithm~\ref{alg:gen-strat-fp}) that
  decides entailment of a fact $\alpha$ by a stratified limit-linear program
  $\Prog$ using oracle $O$ from Step~2. The algorithm maintains a
  pseudo-materialisation $J$, which is initially empty and is constructed
  bottom-up stratum by stratum.  In each step $i$, the algorithm updates the
  pseudo-materialisation by applying $O$ to the union of the
  pseudo-materialisation for stratum $i-1$ and 
  the rules in the $i$-th stratum. The final $J$, from which entailment of
  $\alpha$ is obtained, is computed using 
  a constant number of oracle calls in the size of the data,
  which yields a $\deltatwop$ data complexity upper bound (Proposition
  \ref{prop:algorithm-complexity} and Theorem \ref{thm:strat-lin-deltatwop}).



In what follows, we specify each of these steps.
We start by formally defining the notion of
a \emph{pseudo-materialisation} $\PMat\Prog$ of a stratified limit program $\Prog$, which compactly
represents the  materialisation $\Mat{\Prog}$. Intuitively,
$\Mat{\Prog}$ can be infinite because
it can contain, for any limit predicate $B$ and tuple of objects $\vec a$ of suitable arity, an infinite number of
facts of the form $B(\vec a,k)$. However, if 
the materialisation has facts  of this form, then
either there is a limit value $\ell$ such that $B(\vec a,k) \in \Mat{\Prog}$ for each $k\preceq_B\ell$
and  $B(\vec a,k') \notin \Mat{\Prog}$ for each $k' \succ_B \ell$, or
 $B(\vec a,k) \in \Mat{\Prog}$ for every integer $k$. As argued in prior work,
 it then suffices for the pseudo-materialisation to contain only a single fact $B(\vec a,\ell)$ in the former case, or
 $B(\vec a,\allZ)$ in the latter case.


\begin{definition}\label{def:pseudo-interpretation}
  A \emph{pseudo-interpretation} $J$ is a set of facts such that
  $\allZ$ occurs only in facts over limit predicates and ${k = k'}$ holds
  for all facts $B(\vec a,k)$ and $B(\vec a,k')$ in $J$ with limit $B$.

  The \emph{pseudo-materialisation} of a limit program $\Prog$,
  written $\PMat\Prog$, is the (unique) pseudo-interpretation such that
  \begin{compactenum}
  \item an object or ordinary numeric fact is contained in $\PMat\Prog$ if and
    only if it is contained in $\Mat\Prog$; and
  \item for each limit predicate $B$, object tuple  ${\vec a}$, and
   integer $\ell$,
    \begin{compactitem}
    \item[--] ${B(\vec a,\ell) \in\PMat\Prog}$ if and only if
      ${B(\vec a,\ell) \in\Mat\Prog}$ and ${B(\vec a,k) \not\in\Mat\Prog}$ for
      all ${k \succ_B \ell}$, and
    \item[--] ${B(\vec a,\allZ) \in\PMat\Prog}$ if and only if
      ${B(\vec a,k) \in\Mat\Prog}$ for all integers $k$.
    \end{compactitem}
  \end{compactenum}
\end{definition}

We now strengthen the results in~\cite{DBLP:conf/ijcai/KaminskiGKMH17} by
establishing a bound on the size of pseudo-materialisations of positive,
limit-linear programs.

\begin{restatable}{lemma}{pseudointerpretationsbounded}\label{lem:pseudo-mat-bounded}
  Let\/ $\Prog$ be a semi-ground, positive, limit-linear program, and let
  $\Dat$ be a limit dataset. Then
  $|\PMat{\Prog\cup\Dat}| \leq |\Prog \cup \Dat|$ and the magnitude of each
  integer in $\PMat{\Prog\cup\Dat}$ is bounded polynomially in the
  largest magnitude of an integer in ${\Prog \cup \Dat}$, exponentially in $|\Prog|$,
  and double-exponentially in ${\max_{r \in \Prog} \ssizeu{r}}$, where
  $\ssizeu{r}$ stands for the size of the representation of $r$ assuming that
  all numbers take unit space.
\end{restatable}

By Lemma~\ref{lem:pseudo-mat-bounded}, the pseudo-materialisation 
of $\Prog \cup \Dat$
contains at most linearly many facts; furthermore,
the size of each such fact is 
bounded polynomially once $\Prog$ is considered fixed. Hence, 
the pseudo-materialisation of $\Prog$
can be computed in \fpnp{} in data
complexity, even if $\Prog$ is not semi-ground.

\begin{restatable}{corollary}{materialisationfpnp} \label{cor:lim-lin-mat-fpnp} %
  Let\/ $\Prog$ be a positive, limit-linear program. Then the
  function mapping each limit dataset $\Dat$ to $\PMat{\Prog\cup\Dat}$ is computable in \fpnp{}  in $\ssize{\Dat}$.
\end{restatable}


In our second step, we extend this result to semi-positive programs. For this,
we start by defining the notion of a reduct of a semi-positive limit-linear
program $\Prog$.  The reduct is obtained by first computing a semi-ground
instance $\Prog'$ of $\Prog$ and then eliminating all negative literals in
$\Prog'$ while preserving fact entailment.  Intuitively, negative literals can
be eliminated because they involve only EDB predicates; as a result, their
extension can be computed in polynomial time from the facts in $\Prog$
alone. To eliminate a ground negative literal $\mu$, it suffices to check
whether $\mu$ is entailed by the facts in $\Prog$ and simplify all rules
containing $\mu$ accordingly; in turn, limit literals involving a numeric
variable $m$ can be rewritten as comparisons of $m$ with a constant computed
from the facts in $\Prog$.

\begin{definition}\label{def:reduct}
  Let $\Prog$ be a semi-positive, limit-linear program and let $\Dat$ be the
  subset of all facts in $\Prog$.  The \emph{reduct} of 
  $\Prog$ is obtained by first computing the semi-grounding $\Prog'$ of\/
  $\Prog$ and then applying the following transformations to each rule
  $r\in\Prog'$ and each negative body literal $\mu$ in $r$:
  \begin{compactenum}
  \item if $\mu=\naf\alpha$, for $\alpha$ a ground atom, delete $r$ if  $\Dat \models \alpha$, 
  and delete $\mu$ from $r$ otherwise,
  \item if $\mu=\naf A(\vec a,m)$ is a non-ground limit literal, then
  \begin{compactitem}
    \item[--] delete $r$ if $\Dat \models A(\vec a,k)$ for each integer $k$;
    \item[--] delete $\mu$ from $r$ if
     $\Dat \not\models A(\vec a,k)$ for each $k$; and
    \item[--] replace $\mu$ in $r$ with $(k\prec_A m)$ otherwise, where 
      $\Dat\models\lub{A(\vec a,k)}$.
  \end{compactitem} 
  \end{compactenum}
\end{definition}

Note that
semi-ground programs disallow
non-ground negative literals over ordinary numeric predicates, which
is why these are not considered in Definition~\ref{def:reduct}. As shown by the
following lemma, reducts allow us to reduce fact entailment for
semi-positive, limit-linear programs to semi-ground, positive, limit-linear
programs.

\begin{restatable}{lemma}{reductcorrect} \label{lem:reduct} %
  For $\Prog$ a semi-positive, limit-linear program and $\Dat$ a limit dataset,
  $\Prog'$ the reduct of $\Prog\cup\Dat$, and $\alpha$ a fact, we have
  $\Prog\cup\Dat\models\alpha$ if and only if $\Prog'\models\alpha$. Moreover $\Prog'$ can be computed in polynomial time in $\ssize{\Dat}$,
  $\ssize{\Prog'}$ is polynomially bounded in $\ssize{\Dat}$, and
  ${\max_{r \in \Prog'} \ssizeu{r}\le\max_{r \in \Prog\cup\Dat} \ssizeu{r}}$.
\end{restatable}

The results in Lemma~\ref{lem:pseudo-mat-bounded} and
Lemma~\ref{lem:reduct} imply that the pseudo-materialisation of a
semi-positive, limit-linear program can be computed in \fpnp{} in
data complexity.

\begin{restatable}{lemma}{materialisationfpnpsemipos} \label{lem:lim-lin-mat-fpnp} %
  Let\/ $\Prog$ be a semi-positive, limit-linear program. Then
  the function mapping each limit dataset $\Dat$ to $\PMat{\Prog\cup\Dat}$ is computable in \fpnp{} in $\ssize{\Dat}$.
\end{restatable}


\IncMargin{1em}
\begin{algorithm}[t]
  \DontPrintSemicolon
  \Indentp{-1em}
  \KwPar{oracle $O$ computing 
    $\PMat{\Prog'}$ for $\Prog'$ a semi-positive, limit-linear program}
  \KwIn{stratified, limit-linear program $\Prog$, fact $\alpha$}
  \KwOut{$\mathsf{true}$ if $\Prog \models \alpha$}
  \Indentp{1em}
  compute a stratification $\lambda$ of $\Prog$\nllabel{alg1:strat}\;
  $J \defeq \emptyset$\;
  \For{$i:=1$ \KwTo $\max\qset{k}{\Prog[k]\ne\emptyset}$\nllabel{alg1:loop-begin}}{
    $J \defeq O(\Prog[i]\cup J)$\nllabel{alg1:materialise}\;
    \nllabel{alg1:loop-end}}
  \KwRet $\mathsf{true}$ if $\alpha$ is satisfied in $J$ and $\mathsf{false}$ otherwise\nllabel{alg1:return}
  \caption{ \label{alg:gen-strat-fp}}
\end{algorithm}

We are now ready to present Algorithm~\ref{alg:gen-strat-fp}, which decides
entailment of a fact $\alpha$ by a stratified limit-linear program $\Prog$. The
algorithm uses an oracle $O$ for computing the pseudo-materialisation of a
semi-positive program. The existence of such oracle and its computational
bounds are ensured by Lemma~\ref{lem:lim-lin-mat-fpnp}.
Algorithm~\ref{alg:gen-strat-fp} constructs the pseudo-materialisation
$\PMat{\Prog}$ of $\Prog$ stratum by stratum in a bottom-up fashion. For each
stratum $i$, the algorithm uses oracle $O$ to compute the
pseudo-materialisation of the program consisting of the rules in the current
stratum and the facts in the pseudo-materialisation computed for the previous
stratum. 
Once $\PMat{\Prog}$ has been constructed, entailment of $\alpha$ is checked
directly over $\PMat{\Prog}$.
 
Correctness of the algorithm is immediate by the properties of $O$ and
the correspondence between pseudo-materialisations and materialisations.
Moreover, if oracle $O$ runs in
$\fph{C}$ in data complexity, for some complexity class $C$, then it can only return
a pseudo-interpretation that is polynomially bounded in data complexity;
as a result, Algorithm~\ref{alg:gen-strat-fp} runs in $\pt^C$ since the number
of strata of $\Prog$ does not depend on the input dataset.

\begin{restatable}{proposition}{algorithmcomplexity} \label{prop:algorithm-complexity}
  If oracle $O$ is computable in $\fph{C}$ in data complexity, then
  Algorithm~\ref{alg:gen-strat-fp} runs in $\pt^C$ in data complexity.
\end{restatable}

The following upper bound immediately follows from the correctness of
Algorithm~\ref{alg:gen-strat-fp} and
Proposition~\ref{prop:algorithm-complexity}.

\begin{restatable}{lemma}{limlinupperbound}\label{lem:upper-bound}
For\/ $\Prog$ a stratified,
  limit-linear program and $\alpha$ a fact, deciding\/ $\Prog\models\alpha$ is
  in $\deltatwop$ in data complexity.
\end{restatable}

The matching lower bound is obtained by reduction from the \textsc{OddMinSAT}
problem \cite{DBLP:journals/jcss/Krentel88}. An instance $\mathcal M$ of
\textsc{OddMinSAT} consists of a repetition-free tuple of variables
$\langle x_N,\dots,x_0\rangle$ and a satisfiable propositional formula
$\varphi$ over these variables.
The question is whether the
truth assignment $\sigma$ satisfying $\varphi$ for which the tuple
$\langle\sigma(x_N),\dots,\sigma(x_0)\rangle$ is lexicographically minimal, assuming $\false<\true$, among all satisfying truth assignments of $\varphi$
has $\sigma(x_0)=\true$. 
In our
reduction, $\mathcal M$ is encoded as a dataset $\Dat_{\mathcal M}$ using
object predicates $\mathit{or}$ and $\mathit{not}$ to encode the structure of
$\varphi$ and numeric predicates to encode the order of variables in
$\langle x_N,\dots,x_0\rangle$; a fixed, two-strata program $\Prog_{\mathit{modd}}$ then
goes through all assignments $\sigma$ in the ascending lexicographic order and evaluates the
encoding of $\varphi$ on $\sigma$ until it finds some $\sigma$ that makes
$\varphi$ true; $\Prog_{\mathit{modd}}$ then derives fact $\mathit{minOdd}$ if
and only if $\sigma(x_0)=\true$. Thus, 
$\Prog_{\mathit{modd}}\cup\Dat_{\mathcal M}\models\mathit{minOdd}$ if and only
if $\mathcal M$ belongs to the language of \textsc{OddMinSAT}.
  
\begin{restatable}{theorem}{limlincomplexity} \label{thm:strat-lin-deltatwop}
  For\/ $\Prog$ a stratified, limit-linear program and $\alpha$ a
  fact, deciding\/ $\Prog\models\alpha$ is \deltatwop-complete in data
  complexity. The lower bound holds already for programs with two
  strata.
\end{restatable}


\section{A Tractable Fragment}\label{sec:tractability}

Tractability in data complexity is an important requirement in
data-intensive applications.
In this section, we propose a syntactic restriction on stratified,
limit-linear programs that is sufficient to ensure tractability of fact entailment in data complexity. 
Our restriction extends that of  \emph{type
  consistency} in prior work to account
for negation. The programs in Examples~\ref{ex:shortest-path} and
\ref{ex:closeness-centrality} are type-consistent.

\begin{definition}\label{def:type-consistent}
  A semi-ground, limit-linear rule $r$ is \emph{type-consistent} if
  \begin{compactitem}[--]
  \item each numeric term $t$ in $r$ is of the form
    ${k_0 + \sum_{i=1}^n k_i \times m_i}$ where $k_0$ is an integer and each
    $k_i$, ${1 \leq i \leq n}$, is a nonzero integer, called the
    \emph{coefficient of variable $m_i$ in $t$};

  \item each numeric variable occurs in exactly one standard body
    literal;

  \item each numeric variable in a negative literal is
    guarded;
      
  \item if the head ${A(\vec a, s)}$ of $r$ is a limit atom, then each unguarded
    variable occurring in $s$ with a positive (or negative) coefficient
    also occurs in the body  in a (unique) positive limit literal
    that is of the same (or different, respectively) type (i.e., $\tmin$ vs. $\tmax$) as
    $A$; 
        
  \item for each comparison ${(s_1 < s_2)}$ or ${(s_1 \leq s_2)}$ in $r$, each
    unguarded variable occurring in $s_1$ with a positive (or negative)
    coefficient also occurs in a (unique) positive $\tmin$ (or $\tmax$, respectively)
    body literal, 
    and each unguarded variable occurring in $s_2$ with a positive (or 
    negative) coefficient occurs in a (unique) positive $\tmax$ (or
    $\tmin$, respectively) body literal.
  \end{compactitem}
  A semi-ground, stratified, limit-linear program is \emph{type-consistent} if all of its
  rules are type-consistent. A stratified limit-linear program $\Prog$ is
  \emph{type-consistent} if the program obtained by first semi-grounding
  $\Prog$ and then simplifying all numeric terms as much as possible is
  type-consistent.
\end{definition}

Similarly to type-consistency for positive programs,
Definition \ref{def:type-consistent} ensures that divergence of
limit facts to $\infty$ can be detected in polynomial time when constructing a
pseudo-materialisation (see \cite{DBLP:conf/ijcai/KaminskiGKMH17} for details).
Furthermore, the conditions in Definition \ref{def:type-consistent} have been 
crafted such that the reduct of a semi-positive type-consistent  program (and hence of any intermediate
program considered while materialising
a stratified program) can be trivially rewritten into a positive type-consistent program. For this, it is essential to 
require a guarded use of negation (see third condition in Definition \ref{def:type-consistent}).

%

\begin{restatable}{lemma}{reductpreservestc} \label{lem:reduct-preserves-tc}
  For $\Prog$ a semi-positive, type-consistent program and $\Dat$ a limit dataset, 
  the reduct of $\Prog\cup\Dat$ is polynomially rewritable to a
  positive, semi-ground, type-consistent program $\Prog'$ such that, for each
  fact $\alpha$, $\Prog\cup\Dat\models\alpha$ if and only if $\Prog'\models\alpha$.
\end{restatable}

Lemma~\ref{lem:reduct-preserves-tc} allows us to extend the polytime
algorithm in~\cite{DBLP:conf/ijcai/KaminskiGKMH17} for computing the
pseudo-material\-isation of a positive type-consistent program to
semi-positive programs, thus obtaining a tractable implementation of
oracle $O$ restricted to type-consistent programs. This suffices since
Algorithm~\ref{alg:gen-strat-fp}, when given a type-consistent program
as input, only applies $O$ to type-consistent programs. Thus, by
Proposition~\ref{prop:algorithm-complexity}, we obtain a polynomial time
upper bound on the data complexity of fact entailment for
type-consistent programs with stratified negation. Since plain Datalog
is already \pt-hard in data complexity, this upper bound is
tight.

\begin{restatable}{theorem}{stablecomplexity} \label{thm:ptime-exptime} %
  For\/ $\Prog$ a stratified, type-consistent program and $\alpha$ a
  fact, deciding\/ $\Prog\models\alpha$ is \pt-complete in data
  complexity.
\end{restatable}

Finally, as we show next, our extended notion of type consistency can be
efficiently recognised.

\begin{restatable}{proposition}{typeconsistentprogramschecking}
  Checking whether a stratified, limit-linear program is type-consistent is in $\textsc{LogSpace}$.
\end{restatable}


\section{Conclusion and Future Work}

Motivated by declarative data analysis applications, we have extended the
language of limit programs with stratified negation-as-failure. We have shown
that the additional expressive power provided by our extended language comes at
a computational cost, but we have also identified sufficient syntactic
conditions that ensure tractability of reasoning in data complexity.  There are
many avenues for future work.  First, it would be interesting to formally study
the \emph{expressive power} of our language.  Since type-consistent programs
extend plain (function-free) Datalog with stratified negation, it is clear that
they capture $\pt$ on ordered datasets \cite{DBLP:journals/csur/DantsinEGV01},
and we conjecture that the full language of stratified limit-linear programs
captures $\deltatwop$.
From a more practical perspective, we believe that limit
programs can naturally express many tasks that admit a dynamic programming
solution (e.g., variants of the knapsack problem, and many
others). Conceptually, a dynamic programming approach can be seen as a
three-stage process: first, one constructs an acyclic `graph of subproblems'
that orders the subproblems from smallest to largest; then, one computes a
shortest/longest path over this graph to obtain the value of optimal solutions;
finally, one backwards-computes the actual solution by tracing back in the
graph. Capturing the third stage seems to always require non-monotonic negation
(as illustrated in our path computation example), whereas the first stage may
or may not require it depending on the problem. Finally, the second stage can
be realised with a (recursive) positive program.
Second, our formalism should be extended with aggregate
functions. Although certain forms of aggregation can be simulated using
arithmetic functions and iterating over the object domain by exploiting the
ordering, having aggregation explicitly would allow us to express certain tasks
in a more natural way.
Third, we would like to go beyond stratified negation and investigate the
theoretical properties of limit Datalog under well-founded~\cite{DBLP:journals/jacm/GelderRS91} or the
stable model semantics~\cite{DBLP:conf/iclp/GelfondL88}.
Finally, we plan to implement our reasoning algorithms
and test them in practice.

\section*{Acknowledgments}

This research was supported by the EPSRC projects DBOnto,
MaSI$^3$, and ED$^3$.

\bibliographystyle{named}
\bibliography{references}
 
\clearpage
\onecolumn
\appendix
\counterwithin{theorem}{section}
\renewcommand{\theproposition}{\thesection.\arabic{theorem}}
\renewcommand{\thecorollary}{\thesection.\arabic{theorem}}
\renewcommand{\thelemma}{\thesection.\arabic{theorem}}
\renewcommand{\thedefinition}{\thesection.\arabic{theorem}}
\renewcommand{\theclaim}{\thesection.\arabic{claim}}

\ifdraft{
    \section{Proofs for Section~\ref{SEC:DECIDABILITY}}\label{sec:proofs-opt-progs}


Before proceeding to the proofs of our theorems in the main body of the paper,
we restate some notions from~\cite{DBLP:conf/ijcai/KaminskiGKMH17}. All models
of a limit program are easily seen to satisfy the following closure property.

\begin{definition}\label{def:limit-closed-int}
  An interpretation $I$ is \emph{limit-closed} (for a limit program $\Prog$)
  if, for each fact ${B(\mathbf a,k) \in I}$ where $B$ is a limit predicate,
  ${B(\mathbf a,k') \in I}$ holds for each integer $k'$ with
  ${k' \preceq_B
    k}$. 
\end{definition}

There is a one-to-one correspondence between pseudo-interpretations and
limit-closed interpretations, and thus each model of a program can be
equivalently represented by a pseudo-interpretation.

\begin{definition}\label{def:pseudo-model}
  A limit-closed interpretation $I$ \emph{corresponds} to a
  pseudo-interpretation $J$ if the following conditions hold:
  \begin{itemize}
  \item an object or ordinary numeric fact is contained in $J$ if and only if
    it is contained in $I$; and
  \item for each limit predicate $B$, each tuple of objects ${\vec b}$, and
    each integer $\ell$,
    \begin{inparaenum}[(i)]
    \item ${B(\vec b,k) \in I}$ for all $k$ if and only if
      ${B(\vec b,\allZ) \in J}$, and
    \item ${B(\vec b,\ell) \in I}$ and ${B(\vec b,k) \not\in I}$ for all
      ${k \succ_B \ell}$ and $B$ is a limit predicate if and only if
      ${B(\vec b,\ell) \in J}$.
    \end{inparaenum}
  \end{itemize}
  Let $J$ and $J'$ be pseudo-interpretations corresponding to interpretations
  $I$ and $I'$. Then, $J$ \emph{satisfies} a ground atom $\alpha$, written
  ${J \models \alpha}$, if ${I \models \alpha}$; $J$ is a \emph{pseudo-model}
  of a program $\Prog$, written ${J \models \Prog}$, if ${I \models \Prog}$;
  finally, ${J \sqsubseteq J'}$ holds if ${I \subseteq I'}$.
\end{definition}



\citeA{DBLP:conf/ijcai/KaminskiGKMH17} then define an immediate consequence
operator $\ILFPStepOp\Prog{}$ for positive limit programs that works on
pseudo-interpretations and show that the pseudo-materialisation $\PMat\Prog$ of
a positive limit program $\Prog$ can be computed as the pseudo-interpretation
$\PMatSeq{k}{\infty}$ inductively defined as follows, where $\sup S$, for a set
$S$ of pseudo-interpretations, is the supremum of $S$ w.r.t.~$\sqsubseteq$:
\begin{align*}
  \PMatSeq{i}{0}&=\emptyset
  & \PMatSeq{i}{j+1}&=\ILFPStep{\Prog}{\PMatSeq{i}{j}}
  & \PMatSeq{i}{\infty}&=\sup_{j\in\NN}\PMatSeq{i}{j}
\end{align*}
We call pseudo-interpretations $\PMatSeq{}{i}$ \emph{partial
  pseudo-materialisations} of $\Prog$.

The \conp{} upper bound for fact entailment
in~\cite{DBLP:conf/ijcai/KaminskiGKMH17} is shown by a reduction to validity of
Presburger formulas of a certain shape. We next extend this reduction
$\Pres\Prog$ as given in~\cite{DBLP:conf/ijcai/KaminskiGKMH17} for a
(semi-ground and positive) limit-linear program $\Prog$ to account for datasets
involving $\allZ$.

\begin{definition}\label{def:Presburger-encoding}
  For each $n$-ary object predicate $A$, each $(n+1)$-ary ordinary numeric
  predicate $B$, each $(n+1)$-ary limit predicate $C$, each $n$-tuple of
  objects ${\vec a}$, and each integer $k$, let $\defined{A}{\vec a}$,
  $\defined{B}{\vec a k}$, $\defined{C}{\vec a}$ and $\fin{C}{\vec a}$
  be distinct propositional variables, and let $\val{C}{\vec a}$ a distinct
  integer variable. 

  For $\Prog$ a semi-ground, positive, limit-linear program,
  ${\Pres{\Prog} = \bigwedge_{r \in \Prog} \Pres{r}}$ is the Presburger formula
  where $\Pres{r}$ is the formula (with the same quantifier block as $r$) that
  is obtained by replacing each atom $\alpha$ in $r$ with its encoding
  $\Pres{\alpha}$ defined as follows:
  \begin{itemize}
  \item ${\Pres{\alpha} = \alpha}$ if $\alpha$ is a comparison atom;

  \item ${\Pres{\alpha} = \defined{A}{\vec a}}$ if $\alpha$ is an object
    atom of the form $A(\vec a)$;

  \item ${\Pres{\alpha} = \defined{B}{\vec a k}}$ if $\alpha$ is an ordinary
    numeric atom of the form $B(\vec a,s)$ where $s$ is a ground numeric term
    evaluating to $k$;\footnote{Note that all ordinary numeric atoms in $\Prog$
      have this form since $\Prog$ is semi-ground.} 

  \item
    ${\Pres{\alpha} = \defined{C}{\vec a} \land (\neg \fin{C}{\vec a} \lor s
      \preceq_C \val{C}{\vec a})}$ if $\alpha$ is a limit atom of the form
    ${C(\vec a,s)}$ where $s\ne\allZ$; and

  \item $\Pres{\alpha}=\defined{C}{\vec a} \land \neg \fin{C}{\vec a}$ if
    $\alpha$ is a limit atom of the form ${C(\vec a,\allZ)}$.
  \end{itemize}

  Let $J$ be a pseudo-interpretation, and let $\pass$ be an assignment of
  Boolean and integer variables. Then, $J$ \emph{corresponds} to $\pass$ if all
  of the following conditions hold for all $A$, $B$, $C$, and ${\vec a}$ as
  specified above, for each integer ${k \in \mathbb{Z}}$:
  \begin{itemize}
  \item ${\pval{\defined{A}{\vec a}} = \true}$ if and only if
    ${A(\vec a) \in J}$;

  \item ${\pval{\defined{B}{\vec a k}} = \true}$ if and only if
    ${B(\vec a,k) \in J}$;

  \item ${\pval{\defined{C}{\vec a}} = \true}$ if and only if
    ${C(\vec a,\allZ) \in J}$ or there exists ${\ell \in \mathbb{Z}}$ such
    that ${C(\vec a,\ell) \in J}$;

  \item ${\pval{\fin{C}{\vec a}} = \true}$ and
    ${\pval{\val{C}{\vec a}} = k}$ if and only if ${C(\vec a,k) \in J}$.
  \end{itemize}
\end{definition}

Note that $k$ in Definition~\ref{def:Presburger-encoding} ranges over all
integers (which excludes $\allZ$), ${\pval{\val{C}{\vec a}}}$ is equal
to some integer $k$, and $J$ is a pseudo-interpretation and thus cannot contain
both ${C(\vec a,\allZ)}$ and ${C(\vec a,k)}$; thus, ${C(\vec
  a,\allZ) \in J}$ implies ${\pval{\fin{C}{\vec a}} = \false}$.

The key property of the Presburger encoding
in~\cite{DBLP:conf/ijcai/KaminskiGKMH17} is established by the following lemma,
which we easily re-prove for our variant of the encoding.

\begin{lemma}\label{lemma:pseudo-presburger-correspondence}
    Let $J$ be a pseudo-interpretation and let $\pass$ be a variable assignment
    such that $J$ corresponds to $\pass$. Then,
    \begin{enumerate}
        \item ${J \models \alpha}$ if and only if ${\mu \models
        \mathsf{Pres}(\alpha)}$ for each ground atom $\alpha$, and
        
        \item ${J \models r}$ if and only if ${\mu \models \mathsf{Pres}(r)}$
        for each semi-ground, positive rule $r$.
    \end{enumerate}
\end{lemma}

\begin{proof}
  Claim 1 follows analogously to the respective argument
  in~\cite{DBLP:conf/ijcai/KaminskiGKMH17} except for having an extra case,
  namely ${\alpha = C(\vec a,\allZ)}$, for $C$ a limit predicate. The proof of
  this case is analogous but simpler to the case for ${\alpha = C(\vec a,k)}$
  where $k\in\ZZ$. Claim~2 then follows from Claim~1 same as before.
\end{proof}

Using Lemma~\ref{lemma:pseudo-presburger-correspondence},
\citeA{DBLP:conf/ijcai/KaminskiGKMH17} establish the following correspondence
between entailment for positive limit-linear programs and validity of
Presburger sentences.

\begin{restatable}{lemma}{presburgerencodingold}\label{lem:presburger-encoding-old}
  For $\Prog$ a semi-ground, positive, limit-linear program and $\alpha$ a
  fact, there exists a Presburger sentence
  ${\varphi = \forall \vec x \exists \vec y.\bigvee_{i=1}^n\psi_i}$ that
  is valid if and only if ${\Prog \models \alpha}$. Each $\psi_i$ is a
  conjunction of possibly negated atoms.  Moreover,
  ${|\vec x| + |\vec y|}$ and each $\ssize{\psi_i}$ are bounded
  polynomially by ${\ssize{\Prog} + \ssize{\alpha}}$. Number $n$ is bounded
  polynomially by $|\Prog|$ and exponentially by $\max_{r\in \Prog}
  \ssize{r}$. Finally, the magnitude of each integer in $\varphi$ is bounded by
  the maximal magnitude of an integer in $\Prog$ and $\alpha$.
\end{restatable}

By a more precise analysis of the Presburger formulas in the proof of
Lemma~\ref{lem:presburger-encoding-old}, we can sharpen the bounds provided by
the lemma as follows, where $\ssizeu{r}$ (resp.\ $\ssize{\Prog}$,
$\ssize{\varphi}$, etc.) stands for the size of the representation of $r$
(resp.\ $\Prog$, $\varphi$, etc.) assuming that all numbers take unit space.

\begin{restatable}{lemma}{presburgerencoding}\label{lem:presburger-encoding}
  For\/ $\Prog$ a semi-ground, positive, limit-linear program and $\alpha$ a
  fact, there exists a Presburger sentence
  ${\varphi = \forall \vec x \exists \vec y.\bigvee_{i=1}^n\psi_i}$ that is
  valid if and only if\/ ${\Prog \models \alpha}$. Each $\psi_i$ is a
  conjunction of possibly negated atoms. Moreover, ${|\vec x| + |\vec y|}$ is
  bounded polynomially in $\ssizeu{\Prog}$ and each $\ssizeu{\psi_i}$ is
  bounded polynomially in $\max_{r\in \Prog} \ssizeu{r}$. Number $n$ is bounded
  polynomially in $|\Prog|$ and exponentially in
  $\max_{r\in \Prog} \ssizeu{r}$. Finally, the magnitude of each integer in
  $\varphi$ is bounded by the maximal magnitude of an integer in $\Prog$ and
  $\alpha$.
\end{restatable}

Analogously to the notion of a model for an interpretation, we call With
Lemma~\ref{lem:presburger-encoding-old} at hand,
\citeA{DBLP:conf/ijcai/KaminskiGKMH17} then show the following theorem, which
bounds the magnitude of integers in counter-pseudo-models for entailment (the
proof of the theorem adapts to our setting as is).

\begin{theorem}\label{th:pseudo}
  For $\Prog$ a semi-ground, positive, limit-linear program, $\Dat$ a limit
  dataset, and $\alpha$ a fact, ${\Prog \cup \Dat \not\models \alpha}$ if and
  only if a pseudo-model $J$ of\/ ${\Prog \cup \Dat}$ exists where
  ${J \not\models \alpha}$, ${|J| \leq |\Prog \cup \Dat|}$, and the magnitude
  of each integer in $J$ is bounded polynomially in the largest magnitude of an
  integer in ${\Prog \cup \Dat}$, exponentially in $|\Prog|$, and
  double-exponentially in ${\max_{r \in \Prog} \ssize{r}}$.
\end{theorem}

Furthermore, the double-exponential bound in ${\max_{r \in \Prog} \ssize{r}}$
can be trivially sharpened to ${\max_{r \in \Prog} \ssizeu{r}}$ by employing
Lemma~\ref{lem:presburger-encoding} in place of
Lemma~\ref{lem:presburger-encoding-old}. Building on the proof of
Theorem~\ref{th:pseudo}, we next prove the following stronger version, which
bounds the size of pseudo-materialisations of semi-ground, positive,
limit-linear programs.

\pseudointerpretationsbounded*

\begin{proof}
  Let $a$ be the maximal magnitude of an integer in ${\Prog \cup \Dat}$,
  ${m = |\Prog|}$, and ${n = \max_{r\in\Prog} \ssizeu{r}}$.
  Let $\Dat'$ be obtained from $\Dat$ by removing each fact that does not unify
  with an atom in $\Prog$ and let $E$ be a fresh nullary predicate. 
  
  Clearly, we
  have $\PMat{\Prog\cup\Dat}=\PMat{\Prog\cup\Dat'}\cup J_0$ where $J_0$ is the least
  pseudo-interpretation w.r.t.\ $\sqsubseteq$ such that
  $\set{\alpha}\sqsubseteq J_0$ for each $\alpha\in\Dat\setminus\Dat'$.  Let
  $\varphi$ be obtained from $\Prog\cup\Dat'$ and fact $E$
  analogously to the construction in the proof of
  Lemma~\ref{lem:presburger-encoding}, but where each disjunct
  $(\neg\fin{C}{\vec a} \lor s \preceq_C \val{C}{\vec a})$ in $\Pres\Prog$ is
  replaced by $\neg\fin{C}{\vec a}$ if
  $C(\vec a,\allZ)\in\PMat{\Prog\cup\Dat'}$ and by $s\preceq_C\val{C}{\vec a}$
  if $C(\vec a,k)\in\PMat{\Prog\cup\Dat'}$ for some $k\in\mathbb{Z}$. It is
  easy to see that every assignment corresponding to $\PMat{\Prog\cup\Dat'}$ is
  a countermodel of $\varphi$. Therefore, since $\varphi$ satisfies the same
  structural constraints as the formula in Lemma~\ref{lem:presburger-encoding},
  by an argument analogous to the one in the proof of Theorem~\ref{th:pseudo}
  we obtain that $\Prog\cup\Dat$ has a pseudo-model $J$ such that
  $|J|\le|\Prog\cup\Dat|$, the magnitude of each integer in $J$ is bounded by
  some number $\ell$ that is polynomial in $a$, exponential in $m$, and
  double-exponential in $n$, and where,  it holds that $C(\vec a,\allZ)\in J$ if
  and only if $C(\vec a,\allZ)\in\PMat{\Prog\cup\Dat}$ for each limit predicate $C$ and
  objects $\vec a$. Consequently, we have established that
  $\PMat{\Prog\cup\Dat}$ has a pseudo-model $J$ that satisfies the required
  bounds in the lemma. 
  In what follows we use the fact that $\PMat{\Prog\cup\Dat}
  \sqsubseteq J$ to show that $\PMat{\Prog\cup\Dat}$ also satisfies the bounds
  in the lemma.
   
  Let us denote with $\PMatSeq{}{j}$ the partial pseudo-materialisation of
  $\Prog \cup \Dat$ for any $j\ge 0$ and hence,
  $\PMat{\Prog\cup\Dat}=\PMatSeq{0}{\infty}$.  We start with the observation
  that $(\star)$ the value of a number $k$ in a limit fact $A(\vec{a},k)$ can
  only increase with respect to $\preceq_A$ during the construction of
  $\PMat{\Prog\cup\Dat}$. For instance, if
  $A(\vec{a},k) \in \PMatSeq{}{j}$, with $A$ a $\tmax$ predicate, and
  $A(\vec{a},k') \in \PMatSeq{}{j+1}$, then
  $k' \geq k$.  
  Let, $\ell_0=\ell$ and, for $j>0$, $\ell_j$ be the maximum between
  \begin{compactitem}
  \item $\ell_{j-1}$,
  \item the maximal magnitude of a negative integer occurring in a $\tmax$ fact in
    $\PMatSeq{}{j}$, and
  \item the maximal magnitude of a positive integer occurring in a $\tmin$ fact in
    $\PMatSeq{}{j}$.
  \end{compactitem}
  Numbers $\ell_j$ allow us to bound the integers produced by the immediate
  consequence operator $\ILFPStepOp{\Prog}{}$ applied to pseudo-interpretation
  $\PMatSeq{}{j}$. Specifically, we argue that $(\spadesuit)$ for each
  $j$ and rule $r$ with head $A(\mathbf a,s)$ for some $s$, we have
  \begin{compactitem}
  \item $|\opt{r}{\PMatSeq{i}{j}}|\le n 2^{O(n\log n)}\ell^{n+1}\ell_j$ if
    ${A(\vec a,\allZ)\notin\PMat{\Prog\cup\Dat}}$,
  \item $\opt{r}{\PMatSeq{i}{j}}\ge -n 2^{O(n\log n)}\ell^{n+1}\ell_j$ if $A$ is a $\tmax$
    predicate, and
  \item $\opt{r}{\PMatSeq{i}{j}}\le n 2^{O(n\log n)}\ell^{n+1}\ell_j$ if $A$ is a $\tmin$
    predicate.
  \end{compactitem}
  To see why this holds, consider a pseudo-interpretation $J'$ obtained from
  $\PMatSeq{i}{j}$ by replacing each $\tmax$ IDB fact $B(\vec b,k)$ with
  $B(\vec b,-\ell_j)$, and each $\tmin$ IDB fact $C(\vec c,k')$ with
  $C(\vec c,\ell_j)$. By construction, we have
  $\set{\opt{r}{J'}}\sqsubseteq\set{\opt{r}{\PMatSeq{i}{j}}}\sqsubseteq J$
  and hence
  $\opt{r}{J'}\preceq_A\opt{r}{\PMatSeq{i}{j}}\preceq_A \opt{r}{J}$ whenever
  $\opt{r}{J'}$ is defined. But since the magnitude of all numbers in $J'$ is
  bounded by $\ell_j$, by Proposition~3 in~\cite{DBLP:conf/icalp/ChistikovH16},
  $\constraints{r}{J'}$ has a solution where the maximal magnitude of all
  numbers is bounded by $2^{O(n\log n)}\ell^{n}\ell_j$, and hence the
  magnitude of the value $k$ of $s$ for this solution is bounded by
  $n 2^{O(n\log n)}\ell^{n+1}\ell_j$ (unless the value of $s$ is unbounded in
  $\constraints{r}{J'}$, in which case $\opt{r}{J'}=\opt{r}{J}$ and we are
  done). The last two subclaims are immediate since
  $k\preceq_A\opt{r}{J'}\preceq_A\opt{r}{\PMatSeq{}{j}}$, and and the first
  claim follows since, additionally,
  $\opt{r}{\PMatSeq{}{j}}\preceq_A\opt{r}{J}$, and
  ${A(\vec a,\allZ)\notin\PMat{\Prog\cup\Dat}}$ implies
  $|\opt{r}{J}|\le\ell=\ell_0$ by our assumptions about $J$.
  
  From $(\star)$ we can conclude that, for each $j$ and $k$ such that
  $j\le k$ and $|\PMatSeq{}{j}|=|\PMatSeq{}{k}|$, we have
  $\ell_{j}=\ell_{k}$. Thus, whenever $\ell_j$ increases during the
  construction of $\PMat{\Prog\cup\Dat}$, this must be because a rule has
  generated a fact $B(\vec{b},i)$ where there was previously no fact over $B$
  and $\vec{b}$ in the partial pseudo-materialisation. The number of times this
  can happen is obviously bounded by $m$ (i.e., the number of rules in
  $\Prog$). Furthermore, by $(\spadesuit)$, whenever
  $\ell_{j}<\ell_{j+1}$, we have
  $\ell_{j+1}\le n 2^{O(n\log n)}\ell^{n+1}\ell_j$. Consequently, for every
  $j$, we have that $\ell_j\le n^m 2^{O(m n\log n)}\ell^{m(n+1)+1}$ 

  By ($\spadesuit$), we conclude that the maximal magnitude $L$ of every integer
  in $\PMat{\Prog\cup\Dat}=\PMatSeq{i}{\infty}$ is bounded by
  $n^{m+1}2^{O((m+1)n\log n)}\ell^{(m+1)(n+1)+1}$. Clearly, $L$ is
  polynomially bounded in $a$, exponentially in $m$, and double-exponentially in
  $n$ since so is $\ell$.
\end{proof}

\newcommand{\body}[1]{\mathsf{b}(#1)}

\reductcorrect*

\begin{proof}
  To show $\Prog\cup\Dat\models\alpha$ iff $\Prog'\models\alpha$, it suffices
  to argue that $\Prog'\models\alpha$ holds iff $\Prog''\models\alpha$, for
  $\Prog''$ the semi-grounding of $\Prog\cup\Dat$.

  Since $\Prog''$ is semi-positive and $\Prog'$ positive, w.l.o.g., we have
  $\Mat{\Prog'}=\MatSeq{1}{\infty}$, $\Mat{\Prog''}=\MatSeqq{2}{\infty}$. We
  show that, for each $i\in\NN$,
  \begin{inparaenum}[\it (i)]
  \item 
    $\MatSeq{1}{i}\subseteq\Mat{\Prog''}$, and
  \item 
    $\MatSeqq{2}{i}\subseteq\Mat{\Prog'}$,
  \end{inparaenum}
  by simultaneous induction on $i$, which implies the claim by the definition
  of entailment. Note that, for $r$ a rule, we will denote the body of $r$ as
  $\body{r}$.

  For $i=0$, the claim is trivial since $\MatSeq{1}0=\MatSeqq{2}0=\emptyset$.

  For $i>0$, suppose first $\beta\in\MatSeq{1}{i}$ for some $\beta$. We show
  $\beta\in\Mat{\Prog''}$. Since $\beta\in\IFPStep{\Prog'}{\MatSeq{1}{i-1}}$,
  there is a rule $r'\in\Prog'$ such that, for some grounding $\sigma$,
  $\MatSeq{1}{i-1}\models\body{r'\sigma}$. Moreover, by the inductive
  hypothesis, $\MatSeq{1}{i-1}\subseteq\Mat{\Prog''}$. Let $r''$ be the rule in
  $\Prog''$ such that $r'$ is obtained from $r''$. It suffices to show
  $\Mat{\Prog''}\models\body{r''\sigma}$. By construction, all literals in
  $\body{r'}$ are positive and the only literals in
  $\body{r''}\setminus\body{r'}$ are negative literals of the form
  $\naf\alpha$, so, since $\MatSeq{1}{i-1}\subseteq\Mat{\Prog''}$, it suffices
  to show that $\Mat{\Prog''}\models\naf\alpha\sigma$ for each
  $\naf\alpha\in\body{r''}$. We distinguish two cases.

  If $\alpha$ is ground, we have $\alpha\sigma=\alpha$ and, by construction of
  $r'$, we have $\Dat'\not\models\alpha$, where $\Dat'$ is the set of facts in
  $\Prog\cup\Dat$. Consequently, $\Prog''\not\models\alpha$ since $\Prog''$ and
  $\Dat'$ coincide on facts and $\alpha$ must be EDB in $\Prog''$ (which is the
  case since $\Prog$ is semi-positive), and so
  $\Mat{\Prog''}\not\models\alpha$, and so
  $\Mat{\Prog''}\models\naf\alpha=\naf\alpha\sigma$, as required.

  If $\alpha$ is non-ground, it must be a limit atom of the form $A(\vec a,m)$
  (since $\Prog''$ is semi-ground and thus negative ordinary numeric literals
  contain no numeric variables). 
  By construction of $r'$, one of the following two cases must hold.
  \begin{compactitem}
  \item $\Dat'\not\models A(\vec a,k)$ for each $k\in\ZZ$, and hence
    $\Dat'\not\models A(\vec a,m\sigma)$.
  \item $\Dat'\models\lub{A(\vec a,k)}$ for some $k\in\ZZ$ and
    $(k\prec_A m)\in\body{r'}$. Since $\MatSeq{1}{i-1}\models\body{r'\sigma}$,
    we then have $k\prec_Am\sigma$, and hence
    $\Dat'\not\models A(\vec a,m\sigma)$.
  \end{compactitem}
  Since $A$ must be EDB in $\Prog''$, we then conclude
  $\Mat{\Prog''}\models\naf\alpha\sigma$ analogously to before.

  Next, suppose $\beta\in\MatSeqq{2}{i}$ for some $\beta$. We show
  $\beta\in\Mat{\Prog'}$. Since
  $\beta\in\IFPStep{\Prog''[2]\cup\MatSeqq{1}{\infty}}{\MatSeqq{2}{i-1}}$,
  there is a rule $r''\in\Prog''[2]\cup\MatSeqq{1}{\infty}$ such that, for some
  grounding $\sigma$, $\MatSeqq{2}{i-1}\models\body{r''\sigma}$. Moreover, by
  the inductive hypothesis, $\MatSeqq{2}{i-1}\subseteq\Mat{\Prog'}$.

  We distinguish two cases. If $r''=\beta\in\MatSeqq{1}{\infty}$, it is easily
  seen that $\beta\in\Mat{\Prog'}$ since, by construction, we have
  $\Prog''[1]\subseteq\Prog'$, and hence
  $\MatSeqq{1}{\infty}=\Mat{\Prog''[1]}\subseteq\Mat{\Prog'}$ since $\Prog'$ is
  positive. Thus, w.l.o.g., suppose $r''\in\Prog''[2]$. It then suffices to
  show that there is a rule $r'\in\Prog'$ obtained from $r''$ such that
  $\Mat{\Prog'}\models\body{r'\sigma}$. 
  Since $\MatSeqq{2}{i-1}\models\body{r''\sigma}$, we have
  $\MatSeqq{1}{\infty}\not\models\alpha\sigma$ for each negative literal
  $\naf\alpha\in\body{r''}$, and hence also
  $\Dat'\not\models\alpha\sigma$. Consequently, rule $r''$ is not deleted by
  the transformation rules in Definition~\ref{def:reduct} but rather
  transformed to a positive rule $r'$ such that the only literals in
  $\body{r'}\setminus\body{r''}$ have the form $k\prec_A m$ such that
  $\naf A(\vec a,m)$ is a non-ground limit literal in $\body{r''}$ and
  $\Dat'\models\lub{A(\vec a,k)}$. Thus, since $r'$ is positive and
  $\MatSeqq{2}{i-1}\subseteq\Mat{\Prog'}$, it suffices to show
  $\Mat{\Prog'}\models k\prec_A m\sigma$ for each such literal
  $(k\prec_A m)\in\body{r'}$. This follows since, by construction and since
  $\Prog''$ is semi-positive, $A$ is EDB in $\Prog'$, and hence
  $\Mat{\Prog}\models A(\vec a,s)$ holds for a term $s$ if and only if
  $\Dat'\models A(\vec a,s)$; for each literal $(k\prec_A m)\in\body{r'}$, we
  then have $\Mat{\Prog}\not\models A(\vec a,m\sigma)$ and
  $\Mat{\Prog}\models A(\vec a,k)$, which implies
  $\Mat{\Prog}\models k\prec_A m\sigma$, as required.

  For the second claim, note that, by construction, $\ssize{\Prog'}$
  is bounded from above by $\ssize{\Prog''}$, for $\Prog''$ the
  semi-grounding of $\Prog\cup\Dat$, while $\ssize{\Prog''}$ is easily
  seen to be polynomial in $\ssize{\Dat}$ for $\Prog$ fixed. Moreover,
  $\Prog''$ can be computed in polynomial time, w.r.t.\
  $\ssize{\Dat}$, and each rule in $\Prog'$ can be computed from a
  rule in $\Prog''$ in polynomial time, provided that we can
  polynomially check $\Dat'\models\alpha$, for $\Dat'$ as above. This
  clearly holds since $\Dat'\models\alpha$ can be checked by simply
  matching $\alpha$ against facts in $\Dat'$. Finally
  ${\max_{r \in \Prog'} \ssizeu{r}\le\max_{r \in \Prog''}
    \ssizeu{r}}\le\max_{r \in \Prog\cup\Dat} \ssizeu{r}$.
\end{proof}

We next use Lemmas~\ref{lem:pseudo-mat-bounded} and~\ref{lem:reduct} to show
Lemma~\ref{lem:lim-lin-mat-fpnp}. To this end, we first establish the
following auxiliary result.

\begin{lemma} \label{lem:lim-lin-func-fpnp} Let\/ $\Prog$ be a semi-positive,
  limit-linear program and let\/ $f$ be the function mapping each triple
  $(\Dat,A,\vec a)$, for\/ $\Dat$ a limit dataset, $A$ a $\tmax$ (resp.\
  $\tmin$) predicate and $\vec a$ a tuple of objects, to the greatest (resp.\
  least) $k\in\mathbb Z\cup\set{\allZ}$ such that\/
  $\Prog\cup\Dat\models A(\vec a,k)$ if such $k$ exists, and otherwise to a
  special symbol~$\noneZ$. Then function $f$ is computable in \fpnp{}.
\end{lemma}

\begin{proof}
  Without loss of generality, suppose $A$ is a $\tmax$ predicate. Let $\Prog'$
  be the reduct of $\Prog\cup\Dat$, and let $\ell$ be the bound on the
  magnitude of integers in $\PMat{\Prog'}$ from
  Lemma~\ref{lem:pseudo-mat-bounded}. Then, since, by Lemma~\ref{lem:reduct},
  $\Prog\cup\Dat\models A(\vec a,k)$ implies $A(\vec a,k')\in\PMat{\Prog'}$ for
  some $k'\ge k$, Lemma~\ref{lem:pseudo-mat-bounded} implies that
  $\Prog\cup\Dat\models A(\vec a,\ell+1)$ if and only if
  $\Prog\cup\Dat\models A(\vec a,\allZ)$. Similarly
  $\Prog\cup\Dat\not\models A(\vec a,-\ell)$ if and only if $\Prog\cup\Dat$
  does not satisfy $A(\vec a,k)$ for any $k\in\mathbb Z$. Since $|\Prog'|$ is
  polynomial in $\ssize{\Dat}$ but
  $\max_{r\in\Prog'\setminus\Dat}\ssizeu{r}\le\max_{r\in\Prog}\ssizeu{r}$, by
  Lemma~\ref{lem:pseudo-mat-bounded}, $\ell$ is exponentially bounded in
  $\ssize{\Dat}$, and hence every number in the range of $f$ can be represented
  using polynomially many bits.

  Given a triple $(\Dat,A,\vec a)$, we can thus compute
  $f(\Dat,A,\vec a)$ by a deterministic oracle TM whose oracle set consists
  of all pairs $(\Dat',\alpha)$ such that $\Prog\cup\Dat'\models\alpha$ as follows:
  \begin{enumerate}
  \item Compute the reduct $\Prog'$ of $\Prog\cup\Dat$.
  \item Compute a bound $\ell$ on the magnitude of integers in
    $\PMat{\Prog'}$ satisfying the restrictions in
    Lemma~\ref{lem:pseudo-mat-bounded}.
  \item Perform a binary search for the greatest number $k\in[-\ell,\ell+1]$ such
    that $(\Dat,A(\vec a,k))$ is in the oracle set.
  \item If no such $k$ exists, return $\noneZ$, if $k=\ell+1$, return
    $\allZ$, and otherwise return $k$.
  \end{enumerate}
  Correctness of the algorithm is immediate by the above observations.

  The reduct can be computed in step~(1) in polynomial time and is of
  polynomial size in $\ssize{\Dat}$, whereas
  $\max_{r\in\Prog'}\ssizeu{r}$ is bounded by a constant for a fixed
  $\Prog$ by Lemma~\ref{lem:reduct}. The computation in step~(2)
  takes polynomial time as the binary representation of $\ell$ is
  polynomial in $\ssize{\Dat}$. The search in step~(3) takes
  polynomial time and makes polynomially many oracle calls since the
  interval $[-\ell,\ell+1]$ is exponential in $\ssize{\Dat}$ and does
  not depend on $A$ or $\vec a$, as observed above. Finally, step~(4)
  is clearly polynomial in the size of the input.

  The claim follows since, by the results
  in~\cite{DBLP:conf/ijcai/KaminskiGKMH17}, fact entailment for positive,
  limit-linear programs is \conp{}-complete, hence the membership problem
  for the oracle set is in \conp{}, and $\fpnp=\fph\conp$.
\end{proof}

We then generalise Lemma~\ref{lem:lim-lin-func-fpnp} to
Lemma~\ref{lem:lim-lin-mat-fpnp}.

\materialisationfpnpsemipos*

\begin{proof}
  The set $\PMat{\Prog\cup\Dat}$ can be computed by the following algorithm:
  \begin{enumerate}
  \item Compute the reduct $\Prog'$ of $\Prog\cup\Dat$.
  \item Compute the least (w.r.t.\ $\sqsubseteq$) pseudo-model $J$ of all facts
    in $\Prog'$.
  \item For each IDB predicate $A$ and objects $\vec a$ occurring in the
    head of a rule in $\Prog'$:
    \begin{enumerate}
    \item if $A$ is an object predicate and
      $\Prog\cup\Dat\models A(\vec a)$, add $A(\vec a)$ to $J$;
    \item if $A$ is a $\tmax$ (resp.\ $\tmin$) predicate, compute the greatest
      (resp.\ least) $k\in\ZZ\cup\set{\allZ}$ such that
      $\Prog\cup\Dat\models A(\vec a,k)$, and, if it exists, add
      $A(\vec a,k)$ to $J$.
    \end{enumerate}
  \end{enumerate}
  Correctness of the algorithm follows since $\Prog'$ entails the same facts as
  $\Prog\cup\Dat$ by Lemma~\ref{lem:reduct} and steps (2) and (3) construct the least pseudo-model of
  $\Prog'$. Thus, for the claim, it suffices to show that steps (1), (2), (3.a)
  and (3.b) are all feasible in \fpnp{}, while step (3) is repeated at most
  polynomially often in $\ssize{\Dat}$.
  
  Step~(1) can be performed in polynomial time in $\ssize{\Dat}$ by
  Lemma~\ref{lem:reduct}, while the construction of a pseudo-model of
  a dataset in step~(2) is polynomial in $\ssize{\Prog'}$, and hence
  in $\ssize{\Dat}$, since it involves only trivial
  reasoning. Moreover, step~(3) is repeated at most $|\Prog'|$ times,
  where $|\Prog'|$ is bounded polynomially in $\ssize{\Dat}$ for fixed
  $\Prog$. Finally, step~(3.a) can be performed in \conp{} 
  since fact entailment is \conp{}-complete in data complexity by the results
  in~\cite{DBLP:conf/ijcai/KaminskiGKMH17}, while step~(3.b) is feasible in
  \fpnp{} by Lemma~\ref{lem:lim-lin-func-fpnp}.
\end{proof}

Note that Lemma~\ref{lem:lim-lin-mat-fpnp} immediately implies
Corollary~\ref{cor:lim-lin-mat-fpnp}, so we dispense with a separate
proof for the corollary.

\algorithmcomplexity*

\begin{proof}
  Let $\Prog=\Prog_0\cup\Dat$. Without loss of generality, the number
  of non-empty strata in $\Prog$ is bounded by a constant $s$ for
  $\Prog_0$ fixed, and hence loop~\ref{alg1:loop-begin}--5 is executed
  at most $s$ times. Let $J_i$ be the pseudo-interpretation computed
  by $O$ in iteration $i$ of the loop. By assumption, $O$ in iteration
  $i$ of the loop runs in time bounded by $q(\ssize{\Dat\cup J_{i-1}})$,
  for some polynomial $q$, and hence
  $\ssize{\Dat\cup J_i}\le p(\ssize{\Dat\cup J_{i-1}})$ for some
  polynomial $p$. Consequently, we have
  $\ssize{\PMat\Prog}=\ssize{J_s}\le p(\ssize{\Dat})^s$, which is in
  turn polynomial in $\ssize{\Dat}$, and loop~\ref{alg1:loop-begin}--5
  terminates in time bounded by $s\cdot
  q(p(\ssize{\Dat})^s)$. Finally, step~\ref{alg1:return} can clearly
  be performed in time polynomial in $\ssize{\PMat\Prog}$.
\end{proof}

\limlinupperbound*

\begin{proof}
  The claim is immediate by Lemma~\ref{lem:lim-lin-mat-fpnp} and
  Proposition~\ref{prop:algorithm-complexity}.
\end{proof}

\limlincomplexity*

\begin{proof}
  The upper bound follows by Lemma~\ref{lem:upper-bound} while
  hardness is established by reduction from the \emph{minimal
    satisfying assignment odd} problem. An instance $\mathcal M$ of
  the minimal satisfying assignment odd problem is given by a
  (repetition-free) tuple of variables $\langle x_N,\dots,x_0\rangle$
  and a satisfiable Boolean formula $\varphi$ over $x_0,\dots,x_N$
  (using operators $\lor$ and $\neg$). The problem is to determine
  whether the assignment $\sigma$ for which the tuple
  $\langle\sigma(x_N),\dots,\sigma(x_0)\rangle$ is lexicographically
  minimal (assuming $\false<\true$) among all satisfying truth
  assignments of $\varphi$ satisfies $\sigma(x_0)=\true$. The closely
  related problem where $\langle\sigma(x_N),\dots,\sigma(x_0)\rangle$
  is lexicographically maximal and $\varphi$ is not restricted to be
  satisfiable has been shown $\deltatwop$-complete
  in~\cite[Theorem~3.4]{DBLP:journals/jcss/Krentel88}, and the two
  versions are easily seen to be \logspace{} many-one
  inter-reducible. We reduce the problem by presenting a fixed program
  $\Prog_{\mathit{modd}}$, admitting two strata, and a dataset
  $\Dat_{\mathcal M}$, which depends on $\mathcal M$, and showing that
  $\mathcal M$ is true if and only if
  $\Prog_{\mathit{modd}}\cup\Dat_{\mathcal M}$ entails a nullary fact
  $\mathit{minOdd}$.

  Our encoding uses object EDB predicates $\mathit{root}$, $\mathit{or}$, and
  $\mathit{not}$; ordinary numeric EDB predicate $\mathit{shift}$; and $\tmax$
  IDB predicates $\mathit{ass}$, $\mathit{T}$, and $\mathit{F}$. Program
  $\Prog_{\mathit{modd}}$ consists of
  the following rules, 
  where we write $(s_1\le s_2<s_3)$ as an abbreviation for the conjunction
  $(s_1\le s_2)\land(s_2<s_3)$.
  \begin{align}
    &\to\mathit{ass}(0)\label{eq:minodd:init}\\
    \mathit{root}(x)\land F(x,n)&\to\mathit{ass}(n+1)\label{eq:minodd:inc}\\
    \mathit{or}(x,y,z)\land F(y,n)\land F(z,n)&\to F(x,n)\label{eq:su:orFalse}\\
    \mathit{or}(x,y,z)\land T(y,n)&\to T(x,n)\label{eq:su:orTrue1}\\
    \mathit{or}(x,y,z)\land T(z,n)&\to T(x,n)\label{eq:su:orTrue2}\\
    \mathit{not}(x,y)\land T(y,n)&\to F(x,n)\label{eq:su:notFalse}\\
    \mathit{not}(x,y)\land F(y,n)&\to T(x,n)\label{eq:su:notTrue}\\
    \mathit{ass}(n)\land\mathit{shift}(x,s)\land(0\le m_1)\land(0\le m_2<s)\land(n\doteq 2\times m_1\times s+s+m_2)&\to T(x,n)\label{eq:minodd:varTrue}\\
    \mathit{ass}(n)\land\mathit{shift}(x,s)\land(0\le m_1)\land(0\le m_2<s)\land(n\doteq 2\times m_1\times s+m_2)&\to F(x,n)\label{eq:minodd:varFalse}\\
    \lub{\mathit{ass}(n)}\land(0\le m)\land(n\doteq 2\times m+1)&\to\mathit{minOdd} \label{eq:minodd:minOdd}    
  \end{align}
  Dataset $\Dat_{(\vec x,\varphi)}$ contains facts
  \eqref{eq:minodd:shift}--\eqref{eq:minodd:not}, where, for each distinct
  subformula $\psi$ of $\varphi$ (including $\varphi$ itself), $a_{\psi}$ is a
  fresh object. Note that numbers $2^i$ for $0\le i\le N$ are exponential in
  $N$, and thus can be computed in polynomial time and represented using
  polynomially many bits in the size of the input.
  \begin{align}
    &\to\mathit{shift}(a_{x_i},2^i)&&\textup{for each }0\le i\le N\label{eq:minodd:shift}\\
    &\to\mathit{root}(a_{\varphi})&&\label{eq:minodd:root}\\
    &\to\mathit{or}(a_{\psi},a_{\psi_1},a_{\psi_2})&&\textup{for each subformula $\psi=\psi_1\lor\psi_2$ of $\varphi$}\label{eq:minodd:or}\\
    &\to\mathit{not}(a_{\psi},a_{\psi_1})&&\textup{for each subformula $\psi=\neg\psi_1$ of $\varphi$}\label{eq:minodd:not}
  \end{align}
  In our reduction, each truth assignment $\sigma$ for $x_0,\dots,x_N$ is
  associated with a number $\sum_{0\le i\le N}\sigma(x_i)\times 2^i$. Thus,
  given a number $n$ that encodes a truth assignment, variable $x_i$
  ($0\le i\le N$) is assigned $\true$ if $n=2\times m_1\times 2^i+2^i+m_2$, and
  $\false$ if $n=2\times m_1\times 2^i+m_2$, for some nonnegative integers $m_1$
  and $m_2$ where $m_2<2^i$. Thus, if numeric variable $n$ is assigned such an
  encoding of a truth assignment and numeric variable $s$ is assigned the
  factor $2^i$ corresponding to variable $x_i$, then conjunction
  \begin{gather*}
    (0\le m_1)\land(0\le m_2<s)\land(n\doteq 2\times m_1\times s+s+m_2)
  \end{gather*}
  is true if and only if $x_i$ is true in the assignment (encoded by) $n$;
  analogously, conjunction
  \begin{gather*}
    (0\le m_1)\land(0\le m_2<s)\land(n\doteq 2\times m_1\times s+m_2)
  \end{gather*}
  is true if and only if $x_i$ is $\false$ in assignment
  $n$. Facts~\eqref{eq:minodd:shift} associate with every variable $x_i$ the
  corresponding factor $2^i$, and hence rules~\eqref{eq:minodd:varTrue}
  and~\eqref{eq:minodd:varFalse} derive $T(a_{x_i},n)$ if $x_i$ is true and
  $F(a_{x_i},n)$ if $x_i$ is $\false$ in assignment $n$.
  Facts~\eqref{eq:minodd:root}--\eqref{eq:minodd:not} encode the structure of
  $\varphi$. Using these facts,
  rules~\eqref{eq:su:orFalse}--\eqref{eq:su:notTrue} recursively
  evaluate $\varphi$, deriving, for each subformula $\psi$ of $\varphi$,
  $T(a_\psi,n)$ if $\psi$ evaluates to true and $F(a_\psi,n)$ if $\psi$
  evaluates to $\false$ in assignment $n$. Rules~\eqref{eq:minodd:init}
  and~\eqref{eq:minodd:inc} then search for the lexicographically minimal
  assignment that satisfies $\varphi$ (recall that $\varphi$ is satisfiable by
  assumption)---rule~\eqref{eq:minodd:init} ensures than assignment 0 is
  checked, and rule~\eqref{eq:minodd:inc} ensures that assignment $n+1$ is
  checked whenever $\varphi$ evaluates to $\false$ in $n$.
  Finally, rule~\eqref{eq:minodd:minOdd} derives $\mathit{minOdd}$ if and only
  if $x_0$ is $\true$ in the minimal assignment satisfying $\varphi$, as
  required.
\end{proof}

\section{Proofs for Section~\ref{sec:tractability}}\label{sec:proofs-tc}

\reductpreservestc*

\begin{proof}
  By definition, the program $\Prog'$ obtained by first semi-grounding
  $\Prog\cup\Dat$ and then simplifying all numeric terms as much as
  possible is type-consistent. Thus, it suffices to show that the
  possible violations of type consistency introduced by the additional
  transformation rules in Definition~\ref{def:reduct} can be repaired
  in polynomial time. Since the transformation rules apply to negative
  body literals of an individual rule, suppose $r$ is a semi-ground,
  semi-positive, type-consistent rule and $\mu=\naf\alpha$ is a negative
  body literal of $r$. We have two cases.

  If $\alpha$ is ground, then either $r$ is deleted or $\mu$ is
  deleted from $r$. Clearly, neither of these transformations can
  violate type consistency since $\mu$ does not mention a numeric
  variable.

  If $\alpha=A(\vec a,m)$ is a non-ground limit literal, then one of
  the following is true, for $\Dat'$ the set of facts in $\Prog\cup\Dat$:
  \begin{compactenum}[\it (i)]
  \item $\Dat'\models A(\vec a,k)$ for each $k\in\ZZ$ and $r$ is removed,
  \item $\Dat'\not\models A(\vec a,k)$ for each $k\in\ZZ$ and $\mu$ is
    removed from $r$, or
  \item $\Dat'\models\lub{A(\vec a,k)}$ for some $k\in\ZZ$ and $\mu$
    is replaced in $r$ with $(k\prec_A m)$.
  \end{compactenum}
  Case {\it (i)} does not violate type consistency.

  In case {\it(ii)}, since $r$ is type-consistent, variable $m$ needs
  to be guarded, i.e., there needs to be a conjunction
  $A(\vec a,n)\land(m\doteq n+t)$ for $t\in\set{1,-1}$ in
  $\body{r}$. Moreover, since $A$ is EDB in $\Prog\cup\Dat$, and hence
  in $\Prog'$, $\Dat'\not\models A(\vec a,k)$ implies
  $\Prog'\not\models A(\vec a,k)$ for each $k\in\ZZ$, and thus the
  literal $A(\vec a,n)\in\body{r}$ will never be satisfied when
  computing the materialisation of $\Prog'$, i.e., $r$ is semantically
  redundant and hence can be removed from $\Prog'$, maintaining
  type-consistency.

  Finally, in case {\it(iii)}, type consistency is violated because in
  the transformed rule, variable $m$ no longer occurs in a standard
  body atom. Let $r'$ be the rule obtained from $r$ by a one-step
  application of the transformation rules in case {\it(iii)}.  To
  restore type consistency, we will equivalently re-state rule $r'$
  eliminating all occurrences of $m$. To this end, note that, as in
  the previous case, since $r$ is type-consistent, variable $m$ needs
  to be guarded, i.e., there needs to be a conjunction
  $A(\vec A,n)\land\naf A(\vec a,m)\land(m\doteq n+t)$ for
  $t\in\set{1,-1}$ in $\body{r}$. Thus, let rule $r''$ be obtained
  from $r$ by removing the conjunction
  $A(\vec a,n)\land\naf A(\vec a,m)\land(m\doteq n+t)$ and
  substituting each occurrence of $m$ in $r$ with $k+t$ and each
  occurrence of $n$ with $k$. Clearly, rule $r''$ is type-consistent
  since so is $r$. Moreover, since $A$ is EDB in $\Prog'$, we have
  $\Prog'\models\lub{A(\vec a,k)}$, and hence $r$ can be replaced with
  $r''$ while maintaining the set of entailed facts.

  Clearly, the transformation rules in Definition~\ref{def:reduct}
  restricted to type-consistent programs can be modified to preserve
  type consistency as described above while remaining polynomially
  computable.
\end{proof}

\stablecomplexity*

\begin{proof}
  The \pt{} lower bound in data complexity is inherited from plain
  Datalog~\cite{DBLP:journals/csur/DantsinEGV01}. For the upper bound, note
  that, for $\Prog$ a stratified, type-consistent program, the program
  $\Prog[i]\cup\Dat$ is type-consistent for each $i\in\NN$ and each finite
  dataset $\Dat$. Thus, by Proposition~\ref{prop:algorithm-complexity}, it
  suffices to show that the pseudo-materialisation of a semi-positive,
  type-consistent program $\Prog'$ can be computed in polynomial time in data
  complexity. By Lemmas~\ref{lem:reduct} and~\ref{lem:reduct-preserves-tc},
  this reduces to showing the existence of a polynomial algorithm for computing
  the pseudo-materialisation of a semi-ground, positive, type-consistent
  program $\Prog''$. \citeA{DBLP:conf/ijcai/KaminskiGKMH17} provide such an
  algorithm that terminates in time polynomial in $\ssize{\Prog''}$, provided
  $\max_{r\in\Prog''}\ssizeu{r}$ is bounded by a constant. This assumption can
  be made since, w.l.o.g., $\max_{r\in\Prog'}\ssizeu{r}$ is constant w.r.t.\
  data complexity and, for $\Prog''$ the semi-ground, positive, type-consistent
  program obtained from $\Prog'$ by the results in Lemmas~\ref{lem:reduct}
  and~\ref{lem:reduct-preserves-tc}, we have
  $\max_{r\in\Prog''}\ssizeu{r}\le\max_{r\in\Prog'}\ssizeu{r}$.
\end{proof}

\typeconsistentprogramschecking*

\begin{proof}
  Let $\Prog$ be a stratified, limit-linear program. We can check whether
  $\Prog$ is type-consistent by considering each rule ${r \in \Prog}$
  independently.  For the first type consistency condition, note that each
  maximally simplified numeric term in a semi-ground limit-linear rule has the
  form $k_0+\sum_{i=1}^nk_i\times\prod_{j=1}^{\ell_i}m_i^j$ for all
  $\ell_i\ge 1$. Such a term satisfies the first condition iff, for each $i$,
  either $\ell_i=1$ or $k_i=0$. Thus, to check the first condition, it
  suffices, for each numeric term $s_0+\sum_{i=1}^n s_i$ in $r$ and each $i$
  such that $s_i$ contains at least two variables not occurring in a positive
  ordinary numeric literal, to check that either one of the constants in $s_i$
  is 0 or $s_i$ contains some variable $m$ occurring in a positive ordinary
  numeric literal and 0 is the only constant mentioned in $\Prog$ (and hence
  $m$ must be semi-grounded to 0); clearly, this is doable in logarithmic
  space.

  The second and third conditions are clearly checkable in logarithmic space.

  Thus, it suffices to check whether a semi-grounding of $r$ (with constants
  from $\Prog$) can violate the fourth or the fifth condition. In both cases,
  it suffices to consider at most one atom $\alpha$ at a time (a limit head
  atom ${A(\mathbf a,s)}$ for the fourth condition or a comparison atom
  ${s_1 < s_2}$ or ${s_1 \leq s_2}$ for the fifth condition). In $\alpha$, we
  consider at most one numeric term $s$ at a time (${s \in \set{s_1,s_2}}$ for
  the fifth condition), where, by our considerations for the first condition,
  we can assume w.l.o.g.\ that $s$ has the form
  ${t_0 + \sum_{i=1}^n t_i \times m_i}$ where $t_i$, for ${i \geq 1}$, are
  terms constructed from integers, variables occurring in positive ordinary
  numeric literals, and multiplication. Moreover, for each such $s$, we
  consider each unguarded variable $m$ occurring in $s$.
  By assumption, $m$ occurs in $s$, so we have ${m_i = m}$ for some $i$. For
  the fourth condition of Definition~\ref{def:type-consistent}, we need to
  check that, if the positive limit body literal ${B(\mathbf s,m_i)}$
  introducing $m_i$ (note that $m_i$ cannot be introduced by a negative literal
  by the third condition and since it is unguarded by assumption) has the same
  (different) type as the head atom, then term $t_i$ can only be grounded to
  positive (negative) integers or zero. For the fifth condition, we need to
  check that, if ${s = s_1}$ and the positive limit body literal
  ${B(\mathbf s,m_i)}$ introducing $m_i$ is $\tmin$ ($\tmax$), then term $t_i$
  can only be grounded to positive (negative) integers or 0, and dually for the
  case ${s = s_2}$. Hence, in either case, it suffices to check whether term
  $t_i$ can be semi-grounded so that it evaluates to a positive integer, a
  negative integer, or zero. We next discuss how this can be checked in
  logarithmic space. Let ${t_i = t_i^1 \times \dots \times t_i^k}$, where each
  $t_i^j$ is an integer or a variable not occurring in a limit atom, and assume
  without loss of generality that we want to check whether $t_i$ can be
  grounded to a positive integer; this is the case if and only if one of the
  following holds:
  \begin{itemize}
  \item all $t_i^j$ are integers whose product is positive;

  \item the product of all integers in $t_i$ is positive and $\Prog$ contains a
    positive integer;

  \item the product of all integers in $t_i$ is positive, $\Prog$ contains a
    negative integer, and the total number of variable occurrences in $t_i$ is
    even;

  \item the product of all integers in $t_i$ is negative, $\Prog$ contains a
    negative integer, and the total number of variable occurrences in $t_i$ is
    odd; or

  \item the product of all integers in $t_i$ is negative, $\Prog$ contains both
    positive and negative integers, and some variable $t_i^j$ has an odd number
    of occurrences in $t_i$.
  \end{itemize}
  Each of these conditions can be verified using a constant number of pointers
  into $\Prog$ and binary variables. This clearly requires logarithmic space,
  and it implies our claim.
\end{proof}


}{}

\end{document}
